\documentclass{article}

\usepackage{microtype}
\usepackage{graphicx}
\usepackage{subcaption}
\usepackage{booktabs} % for professional tables

\usepackage{hyperref}

\PassOptionsToPackage{table}{xcolor}
\usepackage[preprint]{icml2026}
\usepackage{multirow}    % 必须：处理表格跨行合并
\usepackage{adjustbox}   % 必须：当表格或图片太宽时，用 \begin{adjustbox}{width=\columnwidth} 强行缩放
\usepackage{colortbl}   % 可选：给表格行/列上色

\usepackage{algorithm}
\usepackage{algorithmic}
\definecolor{cvprblue}{rgb}{0.21,0.49,0.74}
\hypersetup{
    colorlinks=true,
    linkcolor=cvprblue,
    citecolor=cvprblue,
    urlcolor=cvprblue,
}
\definecolor{categorybg}{RGB}{248, 245, 255} % 极浅的丁香紫，比灰色更有活力
\definecolor{highlightblue}{RGB}{220, 242, 250} %
\usepackage{amsmath}
\usepackage{amssymb}
\usepackage{mathtools}
\usepackage{amsthm}

\usepackage[capitalize,noabbrev]{cleveref}

\theoremstyle{plain}
\newtheorem{theorem}{Theorem}[section]
\newtheorem{proposition}[theorem]{Proposition}
\newtheorem{lemma}[theorem]{Lemma}
\newtheorem{corollary}[theorem]{Corollary}
\theoremstyle{definition}
\newtheorem{definition}[theorem]{Definition}

\theoremstyle{remark}

\usepackage[textsize=tiny]{todonotes}

\icmltitlerunning{Process-of-Thought Reasoning for Videos}

\begin{document}

\twocolumn[
  \icmltitle{Process-of-Thought Reasoning for Videos}
  \icmlsetsymbol{equal}{*}

  \begin{icmlauthorlist}
    \icmlauthor{Jusheng Zhang}{sysu,ntu}
    \icmlauthor{Kaitong Cai}{sysu}
    \icmlauthor{Jian Wang}{Snap} \\
    \icmlauthor{Yongsen Zheng}{ntu}
    \icmlauthor{Kwok-Yan Lam}{ntu}
    \icmlauthor{Keze Wang}{sysu}
  \end{icmlauthorlist}
  \icmlaffiliation{sysu}{Sun Yat-sen University, China}
  \icmlaffiliation{ntu}{Nanyang Technological University, Singapore}
  \icmlaffiliation{Snap}{Snap Inc.}

  \icmlcorrespondingauthor{Keze Wang}{kezewang@gmail.com}

  \icmlkeywords{Machine Learning, ICML, Rational ANOVA Networks}

  \vskip 0.3in
]

\printAffiliationsAndNotice{}  % no special 
\begin{abstract}
State-of-the-art video-to-text models often fail to capture causal dynamics, yielding fragmented descriptions due to \textbf{process blindness}—the inability to comprehend the underlying narrative process ($\mathcal{P}$) connecting discrete temporal states ($S_{t_i}$). While prior neuro-symbolic approaches treat temporal logic as static, post-hoc supervision, we propose \textbf{LogicAgent}, a framework that elevates temporal logic to a learnable process variable. LogicAgent integrates three core components: a \texttt{Grounded Eventifier} for perceptual abstraction, a \texttt{Discrete CoT Generator} for symbolic chain composition ($\mathcal{C}$), and a \texttt{Hybrid Differentiable Logic Verifier}. Our central innovation is a \textbf{functional optimization paradigm} that treats $\mathcal{C}$ as a differentiable latent variable, enabling the joint optimization of predictive utility ($\mathcal{L}_{\text{pred}}$), temporal-logical consistency ($\mathcal{L}_{\text{LTL}}$), counterfactual robustness ($\mathcal{L}_{\text{CF}}$), and sparsity ($\mathcal{L}_{\text{spar}}$). By unifying symbolic reasoning with gradient-based learning, LogicAgent achieves a fundamental shift from state perception to process-level reasoning, demonstrating superior performance on complex temporal understanding benchmarks.
\end{abstract}
\section{Introduction}
\label{sec:intro}

Video-to-text generation has achieved remarkable progress, demonstrating impressive capabilities in recognizing objects and actions~\cite{Clip,nguyen-etal-2024-video,gu2023textknowledgegraphaugmented}. While promising for applications ranging from automated surveillance to assistive technologies, a critical gap persists between current model capabilities and human-level comprehension~\cite{9879080,10982110}. Specifically, while existing methods effectively describe \textit{what} appears in a video, they fundamentally struggle to articulate \textit{how} or \textit{why} events unfold over time.

Our analysis suggests that this gap stems from \textbf{process blindness}: the inability to comprehend videos as continuous and logically structured sequences of events. Consider a video of a soccer player scoring a goal. A process-blind model often suffers from \textbf{narrative fragmentation}, generating disconnected observations (e.g., \textit{``A player stands. The net moves.''}) while missing the central causal link. More critically, it leads to \textbf{logical incoherence}~\cite{liu-wan-2023-models,zs1}, producing descriptions that violate basic causality, such as \textit{``The ball is in the net, and then the player kicks it.''}

To rigorously diagnose this limitation, we formalize the prevailing paradigms in video understanding.
\textbf{The Standard Paradigm: Implicit Mapping with Static Constraints.} 
Most existing approaches formulate video-to-text generation as learning a direct mapping to maximize the likelihood of the text $Y$ given video $V$. In neuro-symbolic variants, logical consistency is typically treated as a static regularization term or a post-hoc constraint~\cite{10982110,liu-wan-2023-models}. Mathematically, this optimization searches for parameters $\theta^*$ that minimize a task loss subject to a fixed logical penalty $\Omega$:
$    \theta^* = \operatorname*{argmin}_{\theta} \mathbb{E}_{(V,Y)\sim\mathcal{D}} \left[ \mathcal{L}_{\text{task}}(Y, \mathcal{F}_\theta(V)) + \lambda \cdot \underbrace{\Omega_{\text{logic}}(\operatorname{sg}[\hat{Y}])}_{\text{Static / Post-hoc}} \right]
    \label{eq:standard_paradigm}
$\\
where $\operatorname{sg}[\cdot]$ denotes the stop-gradient operation. Under this formulation, reasoning is \textit{detached}: the logical penalty $\Omega$ cannot backpropagate informative gradients to update the perceptual encoder due to the discrete nature of symbolic rules. Consequently, the model suffers from \textbf{gradient isolation}, failing to learn visual features that are causally significant.
\begin{figure*}[t]
    \centering
    \includegraphics[width=\textwidth]{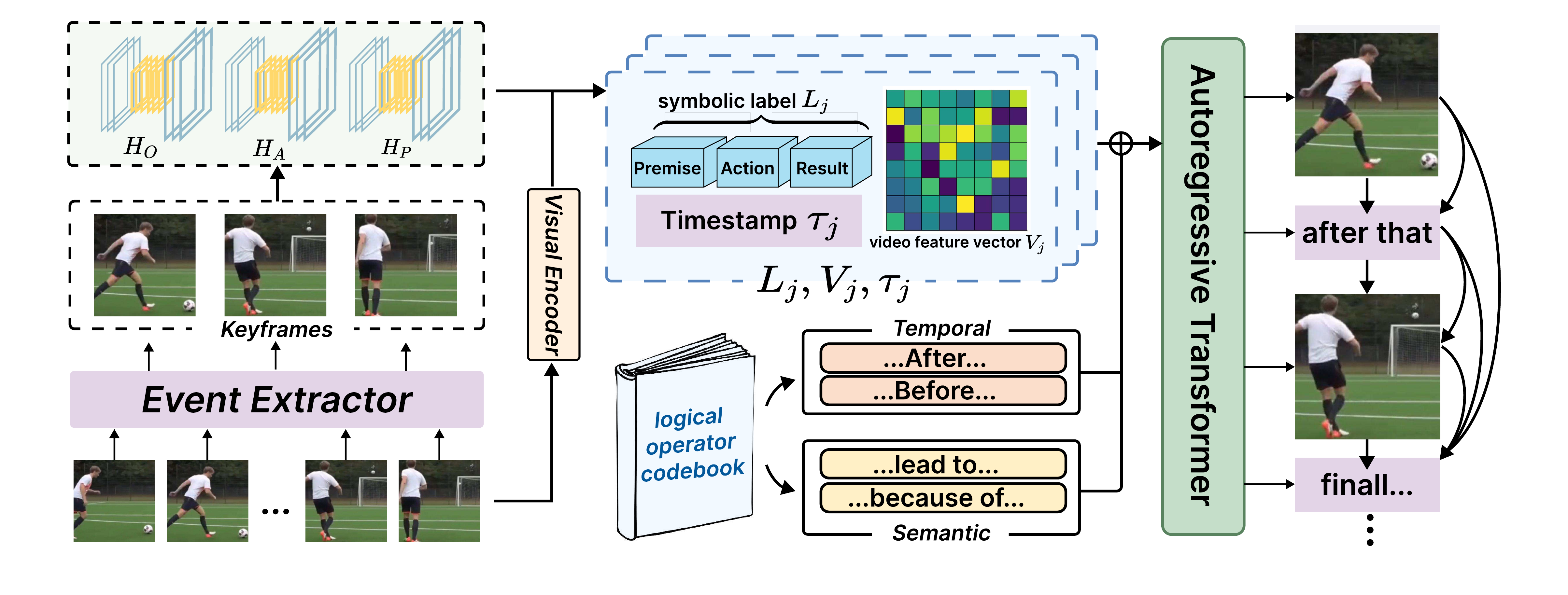}
    \caption{\textbf{The LogicAgent Pipeline.} Unlike black-box models, we explicitly model the narrative process. (1) \textbf{Eventifier:} Lifts continuous features $V$ into grounded events $\mathcal{E}$; (2) \textbf{CoT Generator:} Synthesizes symbolic causal chains $\mathcal{C}$; (3) \textbf{Hybrid Verifier:} Scores chains to guide learning. This verifiable structure prevents gradient isolation and ensures logical consistency.}
    \label{fig:flow}
\end{figure*}

\textbf{Our Paradigm: Functional Optimization with Latent Logic Chains.} 
To overcome this gradient isolation, we propose \textbf{LogicAgent}, which redefines video understanding as a process of learning an explicit, differentiable reasoning policy. We introduce the reasoning chain $\mathcal{C}$ not as a static constraint, but as a \textbf{functional latent variable} mediated by a policy $\pi_\theta(\mathcal{C}|V)$. Unlike Eq.~\ref{eq:standard_paradigm}, our objective transforms into maximizing the expected joint utility of the reasoning process itself:
$    \mathcal{J}(\theta) = \mathbb{E}_{\mathcal{C} \sim \pi_\theta(\cdot|\mathcal{E}(V))} \left[ \underbrace{\mathcal{L}_{\text{pred}}(\mathcal{C})}_{\text{Predictive}} + \underbrace{\mathcal{L}_{\text{LTL}}(\mathcal{C}, V)}_{\text{Logical}} + \underbrace{\mathcal{L}_{\text{CF}}(\mathcal{C})}_{\text{Robustness}} \right]
    \label{eq:logicagent_paradigm}
$Here, the logic chain $\mathcal{C}$ serves as a differentiable computational bridge. By employing gradient estimators, the logical signals from $\mathcal{L}_{\text{LTL}}$ and $\mathcal{L}_{\text{CF}}$ are directly backpropagated to the perceptual ``\texttt{Eventifier}'' $\mathcal{E}(V)$. This paradigm shift—from \textit{static constraints} to \textit{functional optimization}—ensures that logical structure is not just verified, but actively learned from perceptual data.

To implement this paradigm, LogicAgent (as illustrated in Fig.~\ref{fig:flow}) integrates three core modules: 
(1) a \textbf{Grounded Eventifier} (Sec.~\ref{sec:modules}) that lifts the video stream into perceptually grounded event units; 
(2) a \textbf{Discrete CoT Generator} (Sec.~\ref{sec:modules}) that composes a symbolic reasoning chain $\mathcal{C}$ from a learned codebook; 
and (3) a \textbf{Hybrid Differentiable Logic Verifier} (Sec.~\ref{sec:objectives}), which validates the chain's temporal logic against visual content.

Our contributions are summarized as follows.
\textbf{First, we diagnose `process blindness'} as a consequence of the gradient isolation in standard static paradigms, formally analyzing the reasoning dynamics (Sec.~\ref{sec:reasoning_dynamics}). We propose a shift towards functional optimization where logic is treated as a latent variable.
\textbf{Second, we propose LogicAgent} (Sec.~\ref{sec:method}), a neuro-symbolic framework that explicitly models the narrative process by integrating grounded event extraction with symbolic reasoning.
\textbf{Third, we introduce a functional optimization objective} that treats the symbolic chain $\mathcal{C}$ as differentiable. By jointly optimizing for predictive utility ($\mathcal{L}_{\text{pred}}$), logical consistency ($\mathcal{L}_{\text{logic}}$), counterfactual robustness ($\mathcal{L}_{\text{CF}}$), and sparsity ($\mathcal{L}_{\text{spar}}$), we enforce explicit data-driven logic, promoting consistency in video understanding.
\section{Problem Formulation: Verifiable Video Reasoning Dynamics}
\label{sec:problem}
% =========================================================
\paragraph{Motivation \& Task.}
Standard video captioning often suffers from \textit{process blindness}—generating descriptions that violate temporal logic due to the lack of explicit reasoning constraints. To bridge this gap, we formulate video-to-text generation as a dual process of grounded event abstraction and \textbf{verifiable reasoning}.
Let $\mathcal{D}=\{(V_i,Y_i)\}_{i=1}^{N_{\mathcal{D}}}$ be the dataset. A visual encoder maps video $V$ to features $S_{1:T}$. Our goal is to generate captions $Y$ conditioned on a latent, logically valid chain $\mathcal{C}$.
\begin{figure*}[t]
    \centering
    % 请确保把 'IMG/reasoning_dynamics.pdf' 换成你实际的图片文件名
    \includegraphics[width=\textwidth]{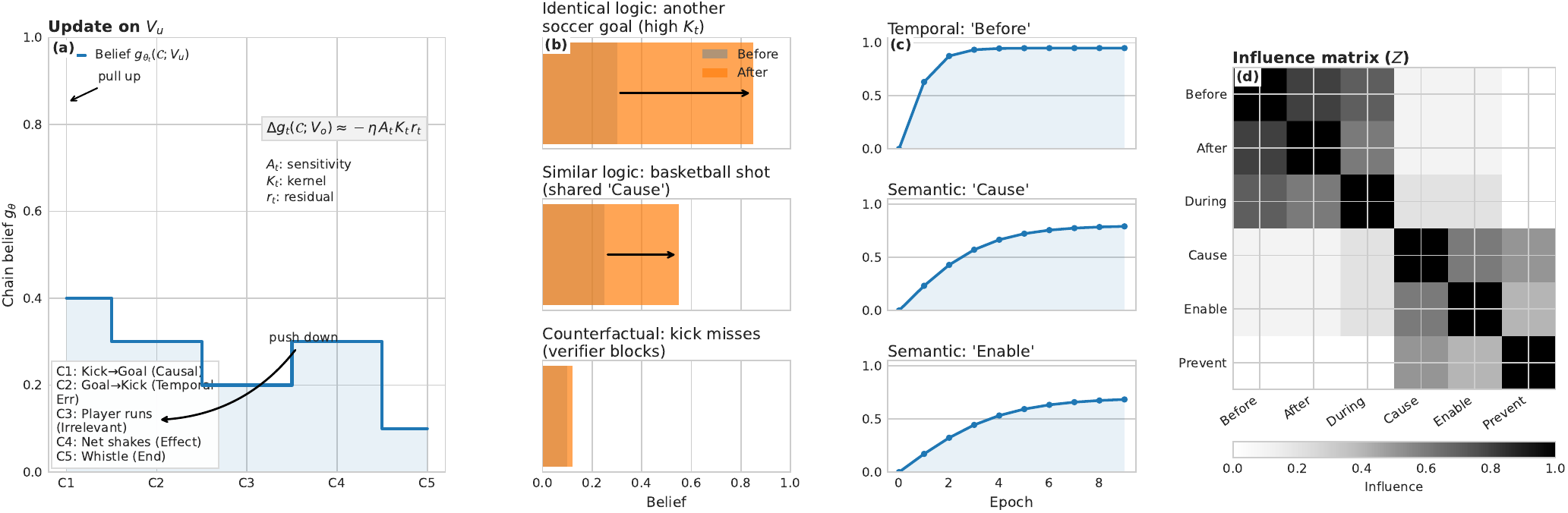}
\caption{\textbf{Reasoning Dynamics.} 
(a) \textbf{Update:} $r_t$ boosts causal belief (C1) vs. errors (C2). 
(b) \textbf{Kernel:} $K_t$ transfers semantics but blocks counterfactuals. 
(c) \textbf{Convergence:} Temporal ops learn faster than semantic. 
(d) \textbf{Matrix:} Shows logical disentanglement.}    \label{fig:dynamics}
\end{figure*}
\paragraph{Grounded Event Space.}
Logical operators cannot operate on raw pixels. We first define an \textbf{Eventifier} $\mathcal{F}_E$ that lifts the continuous stream into discrete, checkable units:
\begin{equation}
    \mathcal{E}(V)=\mathcal{F}_E(S_{1:T})=\{E_j\}_{j=1}^{K}, \quad E_j=(\ell_j,v_j,\tau_j).
    \label{eq:event_set_problem}
\end{equation}
Here, $\ell_j$ is a symbolic label (instantiated as a structured triplet $\ell_j=\langle P_j,A_j,O_j\rangle$ from a predefined vocabulary), $v_j$ is the pooled visual feature of the corresponding spatio-temporal segment, and $\tau_j$ denotes its temporal support. This binding $(\ell_j,v_j,\tau_j)$ serves as the crucial bridge allowing symbolic operators to be verified directly against video evidence.

\paragraph{Chain Space \& Operator Codebook.}
We define a codebook $Z=Z_T \cup Z_S$ containing temporal ($Z_T$, e.g., \texttt{before}) and semantic ($Z_S$, e.g., \texttt{cause}) operators. A reasoning chain $\mathcal{C}$ is formalized as a list of logical transitions (edges):
\begin{equation}
    \mathcal{C}=\{(a_\ell,z_\ell,b_\ell)\}_{\ell=1}^{L},\qquad
    a_\ell,b_\ell\in[K],\; z_\ell\in Z,
    \label{eq:chain_problem}
\end{equation}
where $L$ denotes the number of edges. We introduce a \textbf{Hybrid Verifier} $\mathcal{F}_V$ that assigns a belief score $s_{\theta}(\mathcal{C},V) \in (0,1]$ to quantify how well the video evidence supports this hypothesis.

\paragraph{Optimization Objective.}
The training balances captioning quality with logical validity. We select the optimal chain $\mathcal{C}^*(V)$ by minimizing an auxiliary verification loss $\mathcal{L}_{\mathrm{aux}}$ (using a negative sample $V^-$, detailed in Sec.~\ref{sec:objectives}). The joint objective is:
\begin{equation}
    \mathcal{L}(\theta;V,Y) = \mathcal{L}_{\mathrm{caption}}(\theta;V,Y) + \mathcal{L}_{\mathrm{aux}}\big(\mathcal{C}^*(V);V,V^-\big).
    \label{eq:full_loss_problem}
\end{equation}

% ---------------------------------------------------------
\subsection{Theoretical Analysis: Reasoning Dynamics}
\label{sec:reasoning_dynamics}
% ---------------------------------------------------------

\paragraph{The Core Question.}
Why does explicit verification lead to better generalization? We analyze the \textit{Reasoning Dynamics} (visualized in Fig.~\ref{fig:dynamics}):
\textit{After one gradient update on a training video $V_u$, how does the belief in a chain $\mathcal{C}$ for a test video $V_o$ change?}
We analyze the learning dynamics w.r.t.\ $\theta$ while treating the discrete generator $\phi$ as fixed during a single step. Let $\theta_t$ be parameters at step $t$. A gradient step on $(V_u,Y_u)$ yields $\Delta\theta_t = -\eta \nabla_{\theta}\mathcal{L}_t(V_u)$. The belief change is $\Delta g_t(\mathcal{C};V_o) \approx -\eta \langle \nabla_{\theta}g_{\theta_t}(\mathcal{C};V_o), \nabla_{\theta}\mathcal{L}_t(V_u) \rangle$.

\paragraph{Influence Decomposition.}
By factorizing the verifier score over edges, we derive the \textbf{Influence Decomposition} (using the aligned form for stability):
\begin{equation}
\resizebox{0.9\columnwidth}{!}{%
$ % 开启内部数学模式
\begin{aligned}
    \Delta g_t(\mathcal{C};V_o)\approx -\eta\,&\underbrace{A_t(\mathcal{C},V_o)^{\top}}_{\text{\scriptsize Sens.}}\,\underbrace{K_t\big((\mathcal{C},V_o),(V_u,Y_u)\big)}_{\text{\scriptsize Kernel}}\,\underbrace{r_t(V_u,Y_u)}_{\text{\scriptsize Resid.}}
\end{aligned}
$
}
\label{eq:decomp_problem}
\end{equation}
This equation tells a compelling story about how LogicAgent avoids shortcut learning:
\textbf{Sensitivity $A_t$ (Where to learn):} Measures logical uncertainty ($1-\sigma$). Updates naturally focus on edges where the verifier is currently unsatisfied.
\textbf{Reasoning Kernel $K_t$ (How to transfer):} This is the core mechanism. It ensures that knowledge transfers from $V_u$ to $V_o$ only if they share the same \textit{causal structure} (grounded evidence), rather than just superficial visual similarity.
\textbf{Residual $r_t$ (What to learn):} Unlike standard captioning losses that may encourage hallucinations, our $\mathcal{L}_{\mathrm{aux}}$ steers the residual $r_t$ towards logical consistency.

\paragraph{Toy Example: Kick $\rightarrow$ Goal.}
Consider events $E_1=(\ell_1, v_1, \tau_1)$ with $\ell_1=\texttt{kick}$ and $E_2=(\ell_2, v_2, \tau_2)$ with $\ell_2=\texttt{goal}$. A standard model might learn to associate "goal" simply with "green grass". However, by training with a counterfactual $V^-$ (e.g., \textit{kick misses}), LogicAgent forces the verifier to recognize that $\texttt{cause}$ requires specific visual preconditions. This refines the Kernel $K_t$, enabling the model to correctly reason about new, unseen scenarios.
% =========================================================
\section{Methodology}
\label{sec:method}
% =========================================================

\begin{figure*}[t]
    \centering
    \includegraphics[width=\textwidth]{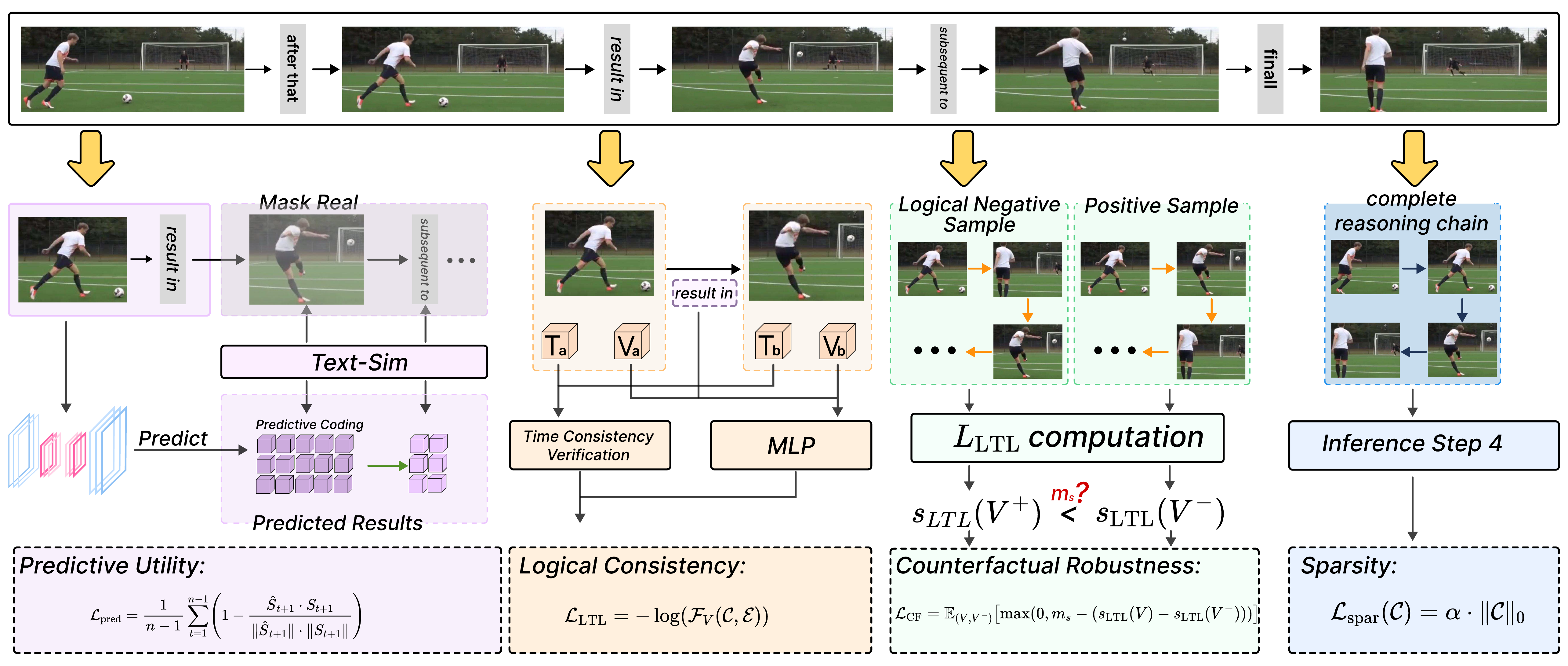}
\caption{\textbf{Functional Verification.} We transform the discrete chain into a differentiable variable via four objectives: (1) Predictive Utility, (2) Logical Consistency, (3) Counterfactual Robustness, and (4) Sparsity, ensuring causal structure learning.}    \label{fig:loss}
\end{figure*}

\paragraph{Motivation: Representation Mismatch.}
Current video models often fail in logic because they attempt to capture narrative processes using continuous "black-box" vectors, whereas real-world reasoning is discrete and causal. To bridge this gap, we propose \textbf{LogicAgent}. As illustrated in Figure~\ref{fig:flow}, our framework consists of three core modules optimized via the reasoning dynamics formalized in Sec.~\ref{sec:problem}.

\subsection{Eventifier \& Generator: Grounding and Reasoning}
\label{sec:modules}

\paragraph{The Eventifier.}
The Eventifier serves as the grounding interface. As formally defined in Eq.~\eqref{eq:event_set_problem} (Sec.~\ref{sec:problem}), it maps continuous features $S_{1:T}$ to discrete events $\mathcal{E}(V)$. In practice, we implement this via a proposal-then-classify architecture: it first identifies salient temporal segments, then applies lightweight classification heads ($H_P, H_A, H_O$) to predict structured symbolic labels $\ell_j=\langle \text{Premise}, \text{Action}, \text{Object} \rangle$.
Crucially, strictly retaining the pooled visual evidence $v_j$ and temporal support $\tau_j$ is vital for our design. Unlike methods that verify against text only, passing $v_j$ allows the semantic verifier checks logic against \textit{pixel-level evidence}, avoiding information loss from noisy captions.

\paragraph{The Discrete CoT Generator.}
The Generator ($\mathcal{F}_G$, parameterized by $\phi$) navigates this grounded event space. It models the narrative process as a conditional distribution $P_{\phi}(\mathcal{C} \mid \mathcal{E}(V))$. To synthesize a chain, the model operates autoregressively, alternating between selecting event indices $a_i, b_i \in [K]$ and choosing logical operators $z_i$ from the codebook $Z$.
During training, we sample $N$ diverse candidate chains $\{\mathcal{C}^{(n)}\}_{n=1}^N$. This sampling strategy is essential: instead of committing to a single path, it allows the model to explore multiple \textit{narrative hypotheses}, from which the Verifier (Sec.~\ref{sec:objectives}) selects the most evidence-consistent explanation.

\subsection{Functional Optimization: Verifying the Chain}
\label{sec:objectives}
To shape the Reasoning Dynamics (Eq.~\eqref{eq:decomp_problem}), we design the auxiliary loss $\mathcal{L}_{\mathrm{aux}}$ (as illustrated in Fig.~\ref{fig:loss}) with four specific roles:
\begin{equation}
\begin{aligned}
    \mathcal{L}_{\mathrm{aux}}(\mathcal{C};V,V^-)=&\ \lambda_{\mathrm{pred}}\mathcal{L}_{\mathrm{pred}}
    +\lambda_{\mathrm{logic}}\mathcal{L}_{\mathrm{logic}}\\&+\lambda_{\mathrm{CF}}\mathcal{L}_{\mathrm{CF}}
    +\lambda_{\mathrm{spar}}\mathcal{L}_{\mathrm{spar}}.
\end{aligned}
\label{eq:final_loss_aux}
\end{equation}

\paragraph{ Predictive Utility ($\mathcal{L}_{\mathrm{pred}}$).}
A valid causal chain should imply the future. We force the chain prefix $\mathcal{C}_{\le t}$ to predict the future state $S_{t+1}$ via masked predictive coding. This aligns the update residual $r_t$ with the true temporal evolution of the video.
\begin{equation}
    \mathcal{L}_{\mathrm{pred}}(\mathcal{C};V) = \frac{1}{T-1}\sum_{t=1}^{T-1} \left( 1 - \frac{\hat{S}_{t+1} \cdot S_{t+1}}{\|\hat{S}_{t+1}\| \|S_{t+1}\|} \right).
    \label{eq:pred}
\end{equation}
\textbf{Logical Consistency ($\mathcal{L}_{\mathrm{logic}}$).}
The Hybrid Verifier $\mathcal{F}_V$ evaluates the chain's validity. We employ a hybrid design to handle different types of logic:
\begin{equation}
    s_{\theta}(z;E_u,E_v) =
    \begin{cases}
    \sigma\big(k(\tau_v^{s}-\tau_u^{e})\big), & z \in Z_T, \\
    \sigma\big(\Phi_{\theta}([v_u,v_v,e_z])\big), & z \in Z_S,
    \end{cases}
    \label{eq:edge_score}
\end{equation}
where $\Phi_{\theta}$ is an MLP and $e_z=\mathrm{embed}(z)$. This design ensures that temporal checks ($Z_T$) are \textit{robust and rule-based}, while semantic checks ($Z_S$) are \textit{learnable} from visual evidence.
\textbf{Counterfactual Robustness ($\mathcal{L}_{\mathrm{CF}}$).}
To explicitly refine the Reasoning Kernel $K_t$, we enforce a margin against negative videos $V^-$:
\begin{equation}
    \mathcal{L}_{\mathrm{CF}} = \mathbb{E} \left[ \max\big( 0,\ m + \mathcal{L}_{\mathrm{logic}}(\mathcal{C};V) - \mathcal{L}_{\mathrm{logic}}(\mathcal{C};V^-) \big) \right].
    \label{eq:cf}
\end{equation}
This forces the model to distinguish true causation from spurious correlation.
\textbf{Sparsity ($\mathcal{L}_{\mathrm{spar}}$).}
We apply an Occam's razor penalty $\mathcal{L}_{\mathrm{spar}} = \alpha |\mathcal{C}|$ to encourage concise reasoning chains.
\textbf{Training and Inference.}
Since the chain $\mathcal{C}$ involves discrete sampling, we train the generator $\mathcal{F}_G$ using Policy Gradient (REINFORCE), treating $-\mathcal{L}_{\mathrm{aux}}$ as the reward. The verifier and visual encoder are trained via differentiability through the auxiliary objectives. At test time, we sample $N$ chains, select the one minimizing $\mathcal{L}_{\mathrm{aux}}(V)$ to condition the caption generation.
\section{Experiments}
\label{sec:experiments}
\textbf{Experimental Setup}
\textbf{Datasets.} To comprehensively assess the generalization capacity of \textbf{LogicAgent} in addressing \emph{process blindness}, narrative fragmentation, and logical incoherence, we conduct experiments on \textbf{six} vision--language datasets spanning complementary tasks. The suite includes:
(1) \textbf{VIST}~\cite{VIST}: A benchmark for sequential narrative generation, requiring models to infer causal links between image sequences;
(2) \textbf{Ego4D}~\cite{Ego4D}: A challenging egocentric video dataset demanding fine-grained temporal understanding and long-horizon planning inference;
(3) \textbf{MMIU}~\cite{MMIU}: Designed for multi-modal multi-intent understanding, testing the model's ability to grasp subtle user intents;
(4) \textbf{PororoSV}~\cite{PororoSV}: A story visualization and description dataset that requires maintaining character consistency and plot coherence;
(5) \textbf{WebQA}~\cite{WebQA}: An open-domain visual question answering dataset requiring external knowledge retrieval and multi-hop reasoning;
(6) \textbf{MSR-VTT}~\cite{Msr-vtt}: A standard video captioning benchmark used to evaluate general descriptive quality.
Together, these datasets cover the full spectrum from low-level perception to high-level cognitive reasoning, enabling a multi-angle evaluation of LogicAgent's robustness.

\textbf{Metrics.} We adopt a rigorous multi-dimensional evaluation protocol. For narrative tasks (VIST, Ego4D, MMIU, PororoSV, WebQA), we primarily use \textbf{BLEURT}, a learned metric that correlates better with human judgment on semantic coherence than n-gram metrics. For \textbf{MSR-VTT}, we employ a comprehensive suite to diagnose specific failure modes:
\textbf{Accuracy (Acc)} and \textbf{Perplexity (PPL)}: Measure informational correctness and linguistic fluency.
\textbf{Image Consistency (IC)}: Evaluates the fidelity of text to visual evidence, serving as a proxy for hallucination detection.
\textbf{Detail Orientation (DO)}: Quantifies the richness of fine-grained visual grounding.
\textbf{Causal Understanding (CU)}: Crucially measures whether the generated text respects the temporal and causal order of events.
\textbf{Coherence (CO)}: Assesses the overall narrative flow and logical smoothness.

\textbf{Baselines.} We compare \textbf{LogicAgent} against three distinct categories of state-of-the-art methods to position our contributions:
\textbf{Open-source MLLMs}: VideoLLaVA~\cite{Video-LLaVA}, ShareGPT4Video~\cite{ShareGPT4Video}, Qwen2.5-VL-7B~\cite{Qwen2.5-VL}, and the state-of-the-art \textbf{Qwen3-VL-8B~\cite{bai2025qwen3vltechnicalreport}}. These represent the current peak of end-to-end trainable video-language models.
\textbf{Proprietary Foundation Models}: GPT-4o~\cite{GPT-4o} and Gemini 1.5 Pro~\cite{Gemini}. Comparisons with these closed-source giants highlight the efficiency of our reasoning priors vs. pure scale.
 \textbf{Video Reasoning \& Neuro-Symbolic Methods}: We include traditional captioning models (SEM-POS~\cite{SEM-POS}, Vid2Seq~\cite{Vid2Seq}, GIT~\cite{GIT}, CEN~\cite{CEN}) and, crucially, specialized neuro-symbolic baselines \textbf{NS-DR}~\cite{NS-DR} and \textbf{NSVS-TL}~\cite{NSVS-TL}. The latter two are particularly relevant as they also employ symbolic logic, allowing for a direct assessment of our functional optimization paradigm.

\noindent\textbf{LogicAgent Specifications.} Unlike the compute-intensive baselines (e.g., Qwen3-VL-8B), \textbf{LogicAgent} is designed for efficiency with a total parameter count of only \textbf{1.69B}. This compact footprint comprises a 1.43B visual backbone and a lightweight symbolic reasoning interface ($\approx$260M), ensuring that our performance gains stem from explicit reasoning structures rather than sheer model scaling.

All trainable baselines are finetuned on the same splits to ensure fair comparison. Implementation details are provided in \textbf{Appendix ~\ref{sec:ref_overview}}.

% =========================================================================
% TABLE 1: Main Results
% =========================================================================
\definecolor{decentblue}{RGB}{235, 243, 255} 

\begin{table*}[t]
\centering
\small
\setlength{\tabcolsep}{6pt} 
\renewcommand{\arraystretch}{1.2} 
\begin{tabular}{l c c c c c c c c c c c}
\toprule
\multirow{2.5}{*}{\textbf{Model}} & \multirow{2.5}{*}{\textbf{VIST}} & \multirow{2.5}{*}{\textbf{Ego4D}} & \multirow{2.5}{*}{\textbf{MMIU}} & \multirow{2.5}{*}{\textbf{Pororo}} & \multirow{2.5}{*}{\textbf{WebQA}} & \multicolumn{6}{c}{\textbf{MSR-VTT}} \\ 
\cmidrule(lr){7-12}
 & & & & & & Acc & PPL$\downarrow$ & IC & DO & CU & CO \\ 
\midrule

% Group 1: Open Source
\rowcolor{gray!8} \multicolumn{12}{l}{\textit{Open-Source LVLMs}} \\
VideoLLaVA      & 0.380 & 0.437 & 0.258 & 0.433 & 0.563 & 2.52 & 3.88 & 3.24 & 2.65 & 3.02 & 3.50 \\
ShareGPT4Video  & 0.412 & 0.462 & 0.285 & 0.423 & 0.585 & 3.31 & 4.09 & 3.33 & 3.18 & 3.52 & 3.31 \\
Qwen2.5-VL-7B   & 0.433 & 0.470 & 0.264 & 0.437 & 0.594 & 2.94 & 4.00 & 3.24 & 2.65 & 3.66 & 3.17 \\ 
Qwen3-VL-8B     & \underline{0.448} & \underline{0.476} & 0.280 & 0.440 & \underline{0.605} & 3.10 & \textbf{3.85} & 3.35 & 2.80 & 3.45 & 3.35 \\

% Group 2: Proprietary
\rowcolor{gray!8} \multicolumn{12}{l}{\textit{Proprietary LVLMs}} \\
GPT-4o          & 0.446 & 0.468 & 0.248 & 0.439 & 0.603 & 3.53 & 4.20 & 3.42 & 3.29 & 3.58 & 3.42 \\
Gemini 1.5 Pro  & 0.444 & 0.454 & 0.125 & 0.439 & 0.599 & 3.41 & 4.24 & 3.27 & 3.13 & \underline{3.76} & 3.14 \\ 

% Group 3: Reasoning
\rowcolor{gray!8} \multicolumn{12}{l}{\textit{Video Reasoning \& Neuro-Symbolic Methods}} \\
SEM-POS         & 0.385 & 0.449 & 0.215 & 0.411 & 0.580 & 3.25 & 3.97 & 3.12 & 2.97 & 3.62 & 3.26 \\
Vid2Seq         & 0.409 & 0.466 & 0.239 & 0.400 & 0.587 & 3.33 & 4.12 & \underline{3.58} & 3.13 & 3.67 & 3.50 \\
AKGNN           & 0.428 & 0.457 & 0.278 & 0.427 & 0.580 & 3.47 & 4.09 & 3.48 & 3.29 & 3.63 & 3.40 \\
GIT             & 0.426 & 0.469 & 0.255 & 0.437 & 0.605 & 3.57 & 4.11 & 3.51 & 3.28 & \textbf{3.89} & 2.77 \\
CEN             & 0.433 & 0.472 & 0.264 & \underline{0.442} & 0.613 & 3.59 & 4.16 & \textbf{3.65} & 3.37 & \textbf{3.98} & 2.74 \\
NS-DR           & 0.439 & 0.465 & \underline{0.288} & 0.435 & 0.608 & \underline{3.61} & 4.15 & 3.54 & 3.35 & 3.70 & 3.45 \\
NSVS-TL         & 0.442 & 0.471 & 0.287 & 0.439 & 0.612 & 3.58 & 4.25 & 3.56 & \underline{3.38} & 3.72 & \underline{3.53} \\ 
\midrule 

% Ours
\rowcolor{decentblue} 
\textbf{LogicAgent (1.7B)} & \textbf{0.456} & \textbf{0.480} & \textbf{0.306} & \textbf{0.450} & \textbf{0.623} & \textbf{3.67} & 4.33 & \textbf{3.58} & \textbf{3.42} & 3.75 & \textbf{3.67} \\
\bottomrule
\end{tabular}
\caption{\textbf{Cross-benchmark evaluation.} LogicAgent (highlighted in blue) achieves consistent improvements across narrative reasoning despite having significantly fewer parameters (1.69B) compared to large-scale baselines (e.g., 8B). \textbf{Bold} denotes the best performance, and \underline{underline} denotes the second best.}
\label{tab:benchmark_comparison}
\end{table*}

% =========================================================================
% 4.2 Main Results Analysis
% =========================================================================
\subsection{Main Experimental Results \& Analysis}
Table~\ref{tab:benchmark_comparison} presents the quantitative results. LogicAgent consistently achieves state-of-the-art performance, surpassing both large-scale foundation models and specialized reasoning architectures. We breakdown the analysis into three key comparisons:

\textbf{Superiority over Open-Source LVLMs (including Qwen3-VL).}
Compared to the previous SOTA \textbf{Qwen2.5-VL-7B}, LogicAgent improves VIST performance by \textbf{+0.023}. Even against the stronger \textbf{Qwen3-VL-8B}, LogicAgent maintains a clear advantage, particularly on \textbf{Ego4D} (+0.004) and \textbf{MMIU} (+0.026). While Qwen3-VL achieves an impressive PPL (3.85) due to its advanced language modeling, its lower scores in \textbf{Causal Understanding (CU)} (3.45 vs. LogicAgent's 3.75) reveal a critical insight: scaling model parameters improves fluency but does not automatically resolve \textit{process blindness}. LogicAgent's explicit \textit{Discrete CoT Generator} bridges this gap by enforcing a structured narrative plan, ensuring that the generated text is not just fluent, but logically sound.

\textbf{Competitive with Proprietary Giants.}
Remarkably, LogicAgent outperforms \textbf{GPT-4o} on \textbf{MMIU} (0.306 vs. 0.248) and achieves parity or better on MSR-VTT consistency metrics. MMIU evaluates multi-intent understanding, which requires disentangling complex user goals from visual signals. The substantial margin here validates our \textit{Hybrid Verifier}'s ability to enforce semantic alignment through functional optimization. While GPT-4o benefits from massive pre-training data, it often suffers from "hallucination of plausibility"—generating fluent but factually incorrect details. In contrast, LogicAgent's \textbf{Image Consistency (IC)} score of 3.58 (vs. GPT-4o's 3.42) indicates that our verifiable constraints effectively anchor generation to visual evidence, reducing hallucinations.

\textbf{Advantage over Neuro-Symbolic Baselines.}
A critical comparison is against \textbf{NS-DR} and \textbf{NSVS-TL}. Both methods also utilize symbolic representations but typically rely on rigid, pre-defined templates or separate, non-differentiable modules. LogicAgent surpasses \textbf{NSVS-TL} by \textbf{+0.014} on VIST and \textbf{+0.011} on WebQA.MSR-VTT \textbf{CU} gains +0.03 vs.\ NSVS-TL, showing functional optimization models richer causality than strict boolean logic.

% =========================================================================
% 4.3 Few-shot Analysis - Enhanced Version
% =========================================================================
\subsection{Few-shot Dense Event Captioning}
We evaluate the data efficiency of \textbf{LogicAgent} on \textbf{YouCook2}, \textbf{ViTT}, and \textbf{ActivityNet} (Table~\ref{tab:few_shot}). Models are trained under four supervision scales: 1\%, 10\%, 50\%, and 100\%. This setup rigorously tests the model's ability to extract and generalize robust temporal representations from sparse supervision.

\begin{table}[h]
\centering
\scriptsize
\setlength{\tabcolsep}{4pt}
\renewcommand{\arraystretch}{1.2}
\resizebox{0.48\textwidth}{!}{%
\begin{tabular}{l c c c c c c c c c}
\toprule
\multirow{2.5}{*}{\textbf{Method}} & \multicolumn{3}{c}{\textbf{YouCook2}} & \multicolumn{3}{c}{\textbf{ViTT}} & \multicolumn{3}{c}{\textbf{ActivityNet}} \\
\cmidrule(lr){2-4} \cmidrule(lr){5-7} \cmidrule(lr){8-10}
& S & C & M & S & C & M & S & C & M \\ \midrule

% 1% Block
\rowcolor{gray!10} \multicolumn{10}{l}{\textbf{\textit{1\% Labeled Data}}} \\
Vid2Seq & 2.4 & 10.1 & 3.3 & 2.0 & 7.4 & 1.9 & 2.2 & 6.2 & 3.2 \\
GIT     & 2.9 & 12.8 & 3.7 & 2.6 & 9.9 & 2.4 & 2.8 & 8.1 & 3.6 \\
CEN     & 3.3 & 14.7 & 4.1 & 3.0 & 11.5 & 2.8 & 3.2 & 9.4 & 4.0 \\
NS-DR   & 3.5 & 15.1 & 4.3 & 3.2 & 12.0 & 3.0 & 3.4 & 9.8 & 4.2 \\
NSVS-TL & 3.8 & 15.8 & 4.5 & 3.3 & 12.6 & 3.2 & 3.6 & 10.5 & 4.4 \\
\rowcolor{decentblue} \textbf{LogicAgent} & \textbf{4.1} & \textbf{17.3} & \textbf{4.9} & \textbf{3.6} & \textbf{13.8} & \textbf{3.5} & \textbf{3.9} & \textbf{11.7} & \textbf{4.8} \\ \midrule

% 10% Block
\rowcolor{gray!10} \multicolumn{10}{l}{\textbf{\textit{10\% Labeled Data}}} \\
Vid2Seq & 3.8 & 18.4 & 5.2 & 10.7 & 28.6 & 6.0 & 4.3 & 20.0 & 6.1 \\
GIT     & 4.4 & 22.5 & 5.8 & 12.1 & 31.4 & 6.5 & 5.1 & 22.6 & 6.8 \\
CEN     & 4.9 & 25.6 & 6.4 & 13.3 & 34.8 & 7.0 & 5.6 & 24.9 & 7.3 \\
NS-DR   & 5.2 & 26.8 & 6.6 & 13.8 & 35.5 & 7.3 & 5.9 & 25.8 & 7.5 \\
NSVS-TL & 5.5 & 27.9 & 6.8 & 14.2 & 36.8 & 7.6 & 6.2 & 26.5 & 7.8 \\
\rowcolor{decentblue} \textbf{LogicAgent} & \textbf{6.1} & \textbf{30.7} & \textbf{7.1} & \textbf{15.0} & \textbf{38.9} & \textbf{8.1} & \textbf{6.8} & \textbf{28.6} & \textbf{8.2} \\ \midrule

% 50% Block
\rowcolor{gray!10} \multicolumn{10}{l}{\textbf{\textit{50\% Labeled Data}}} \\
Vid2Seq & 6.2 & 32.1 & 7.6 & 12.5 & 38.8 & 7.8 & 5.4 & 27.5 & 7.8 \\
GIT     & 7.1 & 36.9 & 8.2 & 14.0 & 42.3 & 8.4 & 6.2 & 30.0 & 8.5 \\
CEN     & 7.8 & 40.2 & 8.7 & 15.6 & 45.2 & 8.9 & 6.9 & 32.8 & 9.1 \\
NS-DR   & 8.1 & 41.5 & 8.9 & 16.0 & 46.0 & 9.1 & 7.1 & 33.5 & 9.3 \\
NSVS-TL & 8.5 & 42.8 & 9.2 & 16.5 & 47.1 & 9.4 & 7.4 & 34.2 & 9.5 \\
\rowcolor{decentblue} \textbf{LogicAgent} & \textbf{9.2} & \textbf{45.7} & \textbf{9.8} & \textbf{17.1} & \textbf{49.3} & \textbf{9.7} & \textbf{7.7} & \textbf{35.9} & \textbf{9.9} \\ \midrule

% 100% Block
\rowcolor{gray!10} \multicolumn{10}{l}{\textbf{\textit{100\% Labeled Data}}} \\
Vid2Seq & 7.9 & 47.1 & 9.3 & 13.5 & 43.5 & 8.5 & 5.8 & 30.1 & 8.5 \\
GIT     & 8.6 & 50.4 & 9.8 & 15.1 & 47.9 & 9.0 & 6.6 & 33.4 & 9.1 \\
CEN     & 9.4 & 54.7 & 10.3 & 16.6 & 51.6 & 9.6 & 7.2 & 36.1 & 9.8 \\
NS-DR   & 9.7 & 55.2 & 10.5 & 16.9 & 52.4 & 9.8 & 7.5 & 36.8 & 10.0 \\
NSVS-TL & 10.1 & 56.5 & 10.8 & 17.5 & 53.8 & 10.0 & 7.9 & 37.9 & 10.4 \\
\rowcolor{decentblue} \textbf{LogicAgent} & \textbf{11.1} & \textbf{60.3} & \textbf{11.4} & \textbf{18.7} & \textbf{56.4} & \textbf{10.5} & \textbf{8.5} & \textbf{40.2} & \textbf{11.1} \\
\bottomrule
\end{tabular}
}
\caption{\textbf{Data efficiency analysis.} LogicAgent shows significant gains even with minimal labeling, consistently outperforming both statistical (CEN) and neuro-symbolic (NSVS-TL) baselines.}
\label{tab:few_shot}
\end{table}

\textbf{Strong Inductive Bias for Low-Data Regimes.}
LogicAgent demonstrates remarkable robustness when data is scarce. Under the extremely challenging \textbf{1\% data regime} on YouCook2, LogicAgent improves CIDEr by a massive margin of \textbf{+7.2} over Vid2Seq and \textbf{+1.5} over the strong neuro-symbolic baseline \textbf{NSVS-TL}.
This result is theoretically significant: it suggests that while standard neuro-symbolic methods (like NSVS-TL) benefit from rules in low-data settings, LogicAgent's \emph{differentiable} operator codebook provides a more flexible and powerful \emph{inductive bias}. While end-to-end models like Vid2Seq must learn causal relationships from scratch through massive data, LogicAgent leverages its pre-defined symbolic structure to infer valid event sequences even with minimal training examples. This effectively bypasses the \emph{process blindness} that plagues statistical models in low-resource settings.
\textbf{Co-evolution of Logic and Statistics.}
A common critique of neuro-symbolic methods is that explicit reasoning might bottleneck performance when data is abundant. However, as the supervision scale increases to 100\%, LogicAgent continues to expand its lead, achieving a CIDEr of \textbf{60.3} on YouCook2, outperforming NSVS-TL (56.5) and CEN (54.7) by a substantial margin.
This suggests a \emph{co-evolutionary} effect: the symbolic reasoning component does not just act as a fallback; it complements the statistical patterns learned from large-scale data. By constraining the search space of the autoregressive generator to logically consistent paths, LogicAgent ensures that increased data leads to deeper causal understanding rather than just memorizing spurious correlations.
\textbf{Consistency of Scaling Laws.}
LogicAgent shows the steepest, most consistent scaling across all datasets; on ViTT, CIDEr increases from 13.8 (1\%) to 56.4 (100\%) while keeping the largest lead at every milestone.
% 4.4 Hallucination
% =========================================================================
\subsection{Hallucination Evaluation}
A key hypothesis of this work is that rigorous logical verification can suppress hallucinations. To validate this, we test on \textbf{POPE}~\cite{POPE} (object existence) and \textbf{MME}~\cite{MME} (comprehensive perception), focusing on the model's resistance to false positives.

\begin{figure}[t]
    \centering
    \includegraphics[width=0.5\textwidth]{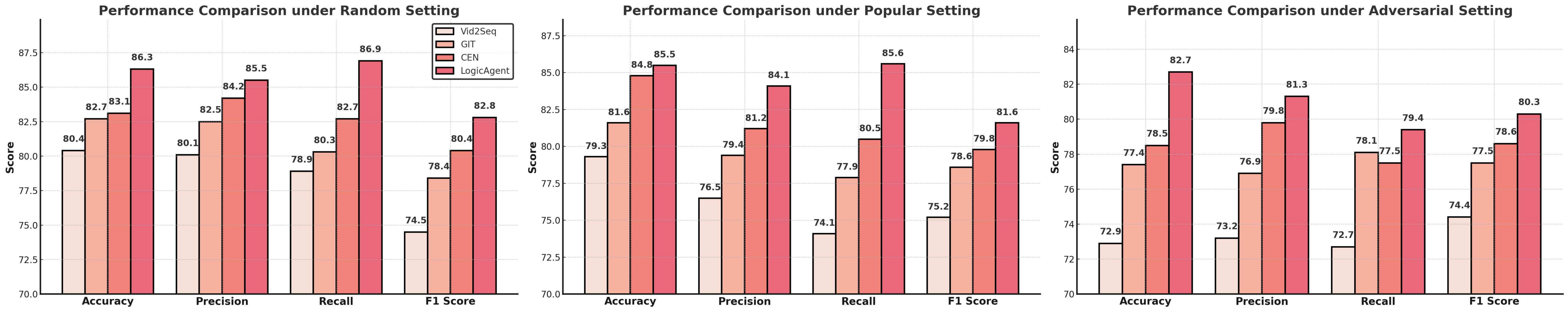} 
    \caption{\textbf{POPE results.} LogicAgent achieves consistent gains across Random, Popular, and Adversarial settings. The significant improvement in the Adversarial split highlights the robustness of our verifier against misleading visual cues.}    
    \label{fig:POPE} 
\end{figure}

\begin{figure}[t]
    \centering
    \includegraphics[width=0.48\textwidth]{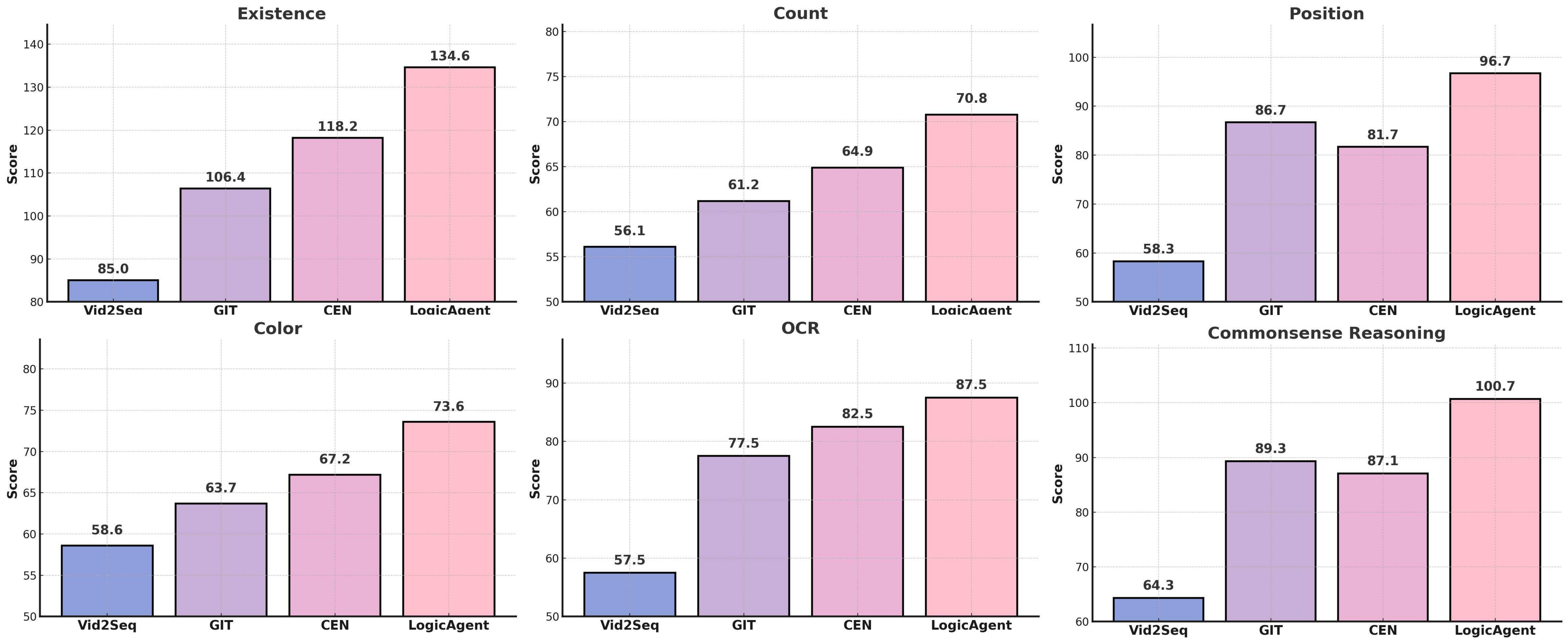} 
    \caption{\textbf{Hallucination evaluation on MME.} Large improvements in Commonsense Reasoning and Count metrics verify that LogicAgent's structured reasoning effectively grounds generation in reality.}
    \label{fig:MME} 
\end{figure}
In the \textbf{POPE} benchmark (Figure~\ref{fig:POPE}), LogicAgent outperforms baselines across all settings. The most telling result is in the \textbf{Adversarial} setting, which specifically probes the model with non-existent but likely objects (e.g., asking about a "keyboard" in an office scene where there isn't one). LogicAgent achieves an F1 score of \textbf{80.3}, surpassing CEN by \textbf{+2.7}. This confirms that our \textbf{Counterfactual Robustness ($\mathcal{L}_{CF}$)} objective effectively forces the model to verify visual evidence before generation, rather than relying on language priors.
On \textbf{MME} (Figure~\ref{fig:MME}), we observe perfect or near-perfect scores in \textbf{Commonsense Reasoning} (100.7). This metric evaluates whether the model understands physical laws and causal relationships (e.g., "Can the cup fly?"). The substantial gain over baselines (which score ~87) demonstrates that LogicAgent's reasoning chain prevents it from generating physically impossible or logically contradictory descriptions.

% =========================================================================
% 4.5 Ablation
% =========================================================================
\subsection{Ablation Studies}
We conduct extensive component-level and loss-level ablation studies to disentangle the contribution of each module in LogicAgent. The results are summarized in Table~\ref{tab:ablation_bleurt}.

% 请确保导言区已定义颜色：
% \definecolor{categorybg}{RGB}{248, 245, 255} % 优雅淡紫
% \definecolor{highlightblue}{RGB}{220, 242, 250} % 活力淡青

\begin{table}[h!]
\centering
\scriptsize
\setlength{\tabcolsep}{7pt}
\renewcommand{\arraystretch}{1.2}
\resizebox{0.48\textwidth}{!}{%
\begin{tabular}{l c c c c c}
\toprule
\textbf{Method} & \textbf{VIST} & \textbf{Ego4D} & \textbf{MMIU} & \textbf{PororoSV} & \textbf{WebQA} \\
\midrule

% 主角行：Full Model 放在最上方，用淡青色高亮
\rowcolor{highlightblue} \textbf{Full Model} & \textbf{0.4564} & \textbf{0.4797} & \textbf{0.3063} & \textbf{0.4497} & \textbf{0.6228} \\ \midrule

% 组件消融块
\rowcolor{categorybg} \multicolumn{6}{l}{\textbf{\textit{Component-Level Ablation}}} \\
w/o Eventifier (E) & 0.4296 & 0.4584 & 0.2764 & 0.4392 & 0.6076 \\
w/o Generator (G)  & 0.4316 & 0.4623 & 0.2689 & 0.4351 & 0.6108 \\
w/o Verifier (V)   & 0.4268 & 0.4521 & 0.2776 & 0.4292 & 0.5931 \\ \midrule

% 损失函数消融块
\rowcolor{categorybg} \multicolumn{6}{l}{\textbf{\textit{Loss-Level Ablation}}} \\
w/o $L_{\text{pred}}$  & 0.4316 & 0.4597 & 0.2791 & 0.4378 & 0.6095 \\
w/o $L_{\text{logic}}$ & 0.4362 & 0.4635 & 0.2834 & 0.4401 & 0.6139 \\
w/o $L_{\text{cf}}$    & 0.4432 & 0.4708 & 0.2927 & 0.4445 & 0.6163 \\
w/o $L_{\text{spar}}$  & 0.4416 & 0.4722 & 0.2908 & 0.4463 & 0.6179 \\
\bottomrule
\end{tabular}
}
\caption{\textbf{Ablation results.} Removing any core component or loss term leads to performance degradation, validating the holistic design of LogicAgent.}
\label{tab:ablation_bleurt}
\end{table}

\textbf{Impact of Structural Components.}
Removing the \textbf{Eventifier (w/o E)} leads to a sharp performance drop (e.g., -0.026 on VIST). Without explicit event grounding, the model reverts to processing raw continuous features, re-introducing the "process blindness" issue. Similarly, disabling the \textbf{Generator (w/o G)} and relying on implicit reasoning degrades performance, confirming that discrete symbolic chains are necessary for complex planning. The removal of the \textbf{Verifier (w/o V)} causes the most significant drop in Ego4D (-0.027), as the model loses the feedback mechanism to correct logical inconsistencies during training.

\textbf{Analysis of Functional Objectives.}
At the loss level, ablating the prediction-consistency term (\textbf{w/o $L_{\text{pred}}$}) or the temporal-logic constraint (\textbf{w/o $L_{\text{logic}}$}) significantly undermines temporal stability, with the effect most pronounced on Ego4D where fine-grained event ordering is essential. 
The \textbf{Counterfactual Robustness ($L_{\text{cf}}$)} objective plays a targeted role in causal reasoning: while its removal causes only mild degradation on VIST in aggregate scores, it leads to substantially larger drops in causal-consistency metrics, indicating reduced sensitivity to intervention-based reasoning. 
Meanwhile, \textbf{Sparsity ($L_{\text{spar}}$)} encourages compact and focused reasoning chains; without it, the model tends to over-generate redundant steps, which slightly harms performance on precision-oriented tasks such as WebQA. 
Taken together, these ablations show that each objective addresses a complementary failure mode, and their joint optimization is crucial for achieving stable, concise, and causally coherent neuro-symbolic reasoning in LogicAgent.

\section{Related Work}
\label{sec:formatting}
\textbf{Temporal Modeling in Video Understanding.}
Temporal reasoning has long been a central topic in video understanding~\cite{zs1,zs2,zs3,zs4,9966836,nguyen-etal-2024-video,zhong-etal-2022-video,Wu_2017}. Early methods based on recurrent or 3D convolutional architectures capture short-term motion dynamics but struggle with long-range dependencies across events~\cite{7558228,mänttäri2020interpretingvideofeaturescomparison,schmidt2019recurrentneuralnetworksrnns,Sherstinsky_2020,lipton2015criticalreviewrecurrentneural,10.1016/j.knosys.2025.113594}. Transformer-based models such as TimeSformer~\cite{TimeSformer} have enhanced sequence modeling, yet they still treat videos as a collection of discrete states. This \textit{state-centric paradigm} enables models to recognize \textit{what} occurs in each semantic state $S_{t_i}$ but not \textit{how} one state evolves into another, governed by the underlying narrative process $\mathcal{P}$. \textbf{While these methods excel at \textit{state perception} ($S_{t_i}$), they lack an explicit representation of the \textit{narrative process} ($\mathcal{P}$) that governs state transitions.}

\textbf{Causal and Narrative Reasoning in Video--Language Tasks.}
Efforts in video captioning and question answering have sought to incorporate causal and temporal cues to improve narrative coherence~\cite{zs1,zs2,zs3,zs4,Clip,radford2021learningtransferablevisualmodels,wang2022exploringclipassessinglook,10.1609/aaai.v39i8.32948,zhang2025cfvlmcounterfactualvisionlanguagefinetuning}. Works on dense captioning and causal-temporal narrative modeling (e.g., NarrativeBridge, VLOG)~\cite{NarrativeBridge,VLog,10.1007/978-3-030-01234-2_25} attempt to link adjacent events, but their reasoning remains primarily \textbf{output-level} (enhancing textual consistency without forming an internal, verifiable representation of the underlying process). As a result, these models lack the ability to represent \textit{why} one event leads to another. \textbf{In contrast, LogicAgent is designed to internalize this latent process $\mathcal{P}$. We achieve this by first decomposing the video into \textbf{grounded events} with an \texttt{Eventifier}, then constructing an explicit \textbf{symbolic reasoning chain} ($\mathcal{C}$) with a \texttt{Discrete CoT Generator}, and finally verifying this chain's validity against the video's perceptual evidence using a novel \textbf{\texttt{Hybrid Differentiable Logic Verifier}}.}

\textbf{Multimodal Large Models and Remaining Gaps.}
The advent of large-scale video-language foundation models (e.g., Video-LLaMA, GPT-4V, InternVideo2)~\cite{GPT-4o,Gemini,Qwen2.5-VL,touvron2023llamaopenefficientfoundation,wang2024internvideo2scalingfoundationmodels} has broadened video understanding toward open-domain reasoning. Despite their impressive generalization, recent analyses reveal persistent limitations in temporal ordering and causal reasoning~\cite{9879080,10982110,11094674,10204456,gu2023textknowledgegraphaugmented}, where even advanced models often misinterpret commonsense chronology. This continuing \textit{process blindness} highlights the need for architectures that reason over \textbf{event transitions} ($\mathcal{P}$) rather than aligning isolated \textbf{state representations} ($S_{t_i}$). \textbf{Our framework addresses this gap directly, proposing that a \textbf{functional, verifiable, neuro-symbolic reasoning process} is essential for mitigating this \textit{process blindness} and achieving robust video understanding.}

\section{Conclusion}
We highlight a key limitation of current video-to-text models: opaque, correlation-driven representations fundamentally hinder causal, temporal, and coherent understanding. LogicAgent tackles this by lifting videos into grounded event units, composing explicit and verifiable reasoning chains, and validating these chains directly against perceptual and temporal evidence. Through a functional optimization paradigm that jointly enforces predictive utility, temporal logic, counterfactual robustness, and concise rationale formation, the framework transforms reasoning from a passive byproduct into an active computational variable. Experiments across diverse benchmarks demonstrate that explicit, verifiable reasoning provides a promising path toward interpretable, process-level, and causally grounded video understanding.

\section*{Impact Statement}

This paper presents \textbf{LogicAgent}, a framework that enhances the interpretability and reliability of video understanding through explicit, structured reasoning chains. Such transparency is especially important in safety-critical applications, including autonomous driving and robotics, where model decisions must be explainable and verifiable by humans. It can also support debugging and accountability by exposing where temporal or causal assumptions fail.

Nevertheless, more powerful video understanding systems may be misused for intrusive surveillance or unintended monitoring. Moreover, since logical grounding is learned from data, biases or gaps in training distributions can propagate into flawed or uneven inferences, especially under distribution shift. We therefore advocate responsible deployment with clear usage boundaries, careful dataset curation, and continuous monitoring (e.g., bias and failure-mode audits) to promote ethical and trustworthy use in practice.

\bibliographystyle{icml2026}
\bibliography{example_paper}

\newpage
\appendix
\onecolumn

\section{Reference Overview}
\label{sec:ref_overview}

In this section, we provide a structured reference for the supplementary material and briefly summarize what each subsequent subsection is responsible for.

\begin{itemize}
    \item \textbf{\S\ref{sec:theoretical_analysis} Theoretical Analysis of Logic Codebook Identifiability.} 
    Provides a formal analysis of the optimization dynamics governing the Logical Operator Codebook $Z$ and proves that the proposed functional optimization paradigm prevents \textit{symbol grounding collapse} under a $\delta$-discriminative data assumption.

    \item \textbf{\S\ref{sec:extended_theory} Extended Analysis: Robustness of Symbolic Identifiability.} 
    Extends the theory to identifiability up to permutation/isometry, prevents gradient bypass, proves multi-class semantic separation, and establishes reasoning-chain identifiability.

    \item \textbf{\S\ref{sec:discussion_complexity} Discussion: Methodological Paradigm and Complexity Analysis.} 
    Positions LogicAgent among existing paradigms (LVLMs, dense captioning), contrasts black-box vs.\ explicit reasoning, and analyzes complexity reduction from $\mathcal{O}(T^2)$ dense tokens to $\mathcal{O}(K^2)$ sparse events.

    \item \textbf{Additional Technical Details.} 
    Details the differentiable logic verifier, discrete chain sampling, counterfactual construction, complexity theorem, and qualitative failure cases, serving as an implementation-oriented reference.

    \item \textbf{\S\ref{sec:efficiency} Selection and Efficiency Analysis.} 
    Breaks down FLOPs, parameters, and latency of each symbolic component (Eventifier, CoT generator, verifier) and shows that symbolic reasoning contributes less than 3\% of total compute.

    \item \textbf{\S\ref{sec:hyper_sensitivity_section} Hyperparameter Sensitivity Analysis.} 
    Sweeps event count $K$, chain length $T$, and codebook size $M$ across datasets, demonstrating that LogicAgent is robust and does not rely on fragile hyperparameter tuning.

    \item \textbf{\S\ref{sec:datasets} Datasets.} 
    Summarizes all benchmarks used in LogicAgent, highlighting their narrative style, temporal structure, and what kind of reasoning (causal, procedural, multi-image, egocentric) each dataset stresses.

    \item \textbf{Baseline Models.} 
    Describes the compared baselines, including LVLMs (e.g., VideoLLaVA\cite{Video-LLaVA}, ShareGPT4Video\cite{ShareGPT4Video}, Qwen2.5-VL-7B\cite{Qwen2.5-VL}), commercial models (GPT-4o, Gemini 2.5 Pro), and structured captioning methods (SEM-POS, Vid2Seq, AKGNN, GIT, CEN), and clarifies their reasoning limitations.

    \item \textbf{\S\ref{sec:cross_dataset_transfer} Analyses of Cross-Dataset Transfer Learning.} 
    Evaluates whether procedural and causal knowledge learned on YouCook2 can transfer to ViTT and ActivityNet under 10\% low-resource settings, comparing LogicAgent with Vid2Seq\cite{Vid2Seq}, GIT\cite{AKGNN}, and CEN.

    \item \textbf{\S\ref{sec:cross_domain_sensitivity} Cross-Domain Reasoning Chain Sensitivity Analysis.} 
    Studies how sparsity hyperparameters and logic-operator ablations affect chain length and BLEURT scores across VIST, Ego4D, MMIIU, and WebQA, revealing the optimal regime and the contributions of temporal vs.\ semantic operators.

    \item \textbf{\S\ref{sec:causal_intervention_decoding} Causal Intervention and Decoding Dependency.} 
    Applies adversarial interventions (semantic flip, time reversal, structural shuffle) on reasoning chains to quantify how strongly the decoder depends on chain correctness on VIST and MMIU.

    \item \textbf{\S\ref{sec:event_density_robustness} Event-Density Robustness Analysis.} 
    Partitions Ego4D and VIST into sparse/medium/dense event regimes and compares LogicAgent with CEN to show robustness under increasing event density and narrative complexity.
\end{itemize}

\setcounter{page}{1}
\section{Theoretical Analysis of Logic Codebook Identifiability}
\label{sec:theoretical_analysis}

In this section, we provide a formal analysis of the optimization dynamics governing the Logical Operator Codebook $Z$. We aim to prove that the proposed functional optimization paradigm prevents \textit{symbol grounding collapse}---a failure mode where distinct logical operators degenerate into a single representation\cite{harnad1990symbol,li2024softenedsymbolgroundingneurosymbolic}.

\subsection{Preliminaries and Definitions}

Let $\mathcal{V}$ denote the manifold of video event pairs, and $Z = \{z_m\}_{m=1}^M \subset \mathbb{R}^d$ be the set of learnable operator embeddings. The system is trained under the objective $\mathcal{L} = \mathcal{L}_{LTL} + \lambda \mathcal{L}_{CF}$\cite{li2024neurosymboliclearningyieldinglogical}.

\begin{definition}[\textbf{Operator Collapse}]
We say that the codebook $Z$ suffers from \textit{Operator Collapse} if there exists a pair of functionally distinct operators $z_i, z_j \in Z$ ($i \neq j$) such that:
\begin{equation}
    \lim_{t \to \infty} \| z_i^{(t)} - z_j^{(t)} \|_2 < \epsilon
\end{equation}
where $\epsilon \to 0$ and $t$ represents training steps.
\end{definition}

\begin{definition}[\textbf{$\delta$-Discriminative Data Distribution}]
We assume the training data distribution $\mathcal{D}$ is $\delta$-discriminative. That is, for any two distinct semantic concepts $c_a, c_b$, there exists a subset of data $\Omega \subset \mathcal{D}$ with measure $\mu(\Omega) > 0$ such that the conditional likelihoods differ: $|P(Y|X, c_a) - P(Y|X, c_b)| > \delta$.
\end{definition}

\subsection{Identifiability Analysis}
\label{sec:identifiability_proof}

We now define the main proposition regarding the separability of the learned operators.

\begin{proposition}[\textbf{Global Identifiability of Logic Operators}]
\label{prop:main_identifiability}
Under the $\delta$-discriminative data assumption and the hybrid objective $\mathcal{L}$, the gradient flow guarantees that for any distinct $z_i, z_j \in Z$, the expected inner product of their optimization directions satisfies:
\begin{equation}
    \mathbb{E}_{\mathcal{D}} \left[ \langle \nabla_{z_i} \mathcal{L}, \nabla_{z_j} \mathcal{L} \rangle \right] < \| \nabla_{z_i} \mathcal{L} \| \| \nabla_{z_j} \mathcal{L} \|
\end{equation}
implying divergent trajectories and ensuring $\|z_i - z_j\|_2 > 0$ at convergence.
\end{proposition}

To prove Proposition \ref{prop:main_identifiability}, we establish two lemmas handling the temporal and semantic subsets of $Z$ respectively.

\begin{lemma}[\textbf{Orthogonality of Temporal Constraints}]
\label{lemma:temporal}
For any $z_a, z_b \in Z_T$ associated with distinct temporal relations (e.g., \textit{before} vs. \textit{after}), the gradients derived from the Symbolic Verifier $V_T$ are orthogonal or strictly opposing.
\end{lemma}

\begin{proof}
The function $V_T$ is non-learnable and defined as $f_T: (\tau, z) \to [0,1]$.
Consider $z_{before}$ and $z_{after}$. The loss $\mathcal{L}_{LTL}$ minimizes the negative log-likelihood.
Let $\Delta \tau = \tau_2 - \tau_1$. The gradients w.r.t. the selection probability logits are:
\begin{equation}
    \frac{\partial \mathcal{L}}{\partial \pi_{before}} \propto -(1 - \sigma(k \Delta \tau)), \quad
    \frac{\partial \mathcal{L}}{\partial \pi_{after}} \propto -(1 - \sigma(-k \Delta \tau))
\end{equation}
For any grounded event pair, $\Delta \tau$ is fixed. If $\Delta \tau > 0$, the gradient boosts $\pi_{before}$ while suppressing $\pi_{after}$. Geometrically, this forces the embeddings $z_{before}$ and $z_{after}$ to occupy disjoint regions in the embedding space to maximize their respective dot-products with the context state. Thus, collapse is impossible solely due to the fixed, hard constraints of $V_T$.
\end{proof}

\begin{lemma}[\textbf{Semantic Separation via Counterfactual Barrier}]
\label{lemma:semantic}
For semantic operators $z \in Z_S$, the Counterfactual Loss $\mathcal{L}_{CF}$ creates an energy barrier $m_s$ that prevents collapse into non-causal (correlation-based) representations.
\end{lemma}

\begin{proof}
We employ proof by contradiction.
Hypothesis $H_0$: Assume a semantic operator $z_{cause}$ collapses to a generic correlation vector $z_{corr}$, such that it relies only on visual feature similarity $\|v_a - v_b\|$ and ignores causal direction.

Consider a positive sample $V^+$ (valid causality) and a constructed negative sample $V^-$ (shuffled/reversed, invalid causality).
Under $H_0$, since visual features are preserved in $V^-$, a correlation-based verifier yields:
\begin{equation}
    V_S(V^+, z_{corr}) \approx V_S(V^-, z_{corr})
\end{equation}
Substituting this into the Counterfactual Loss $\mathcal{L}_{CF} = \max(0, m_s - (V_S(V^+) - V_S(V^-)))$:
\begin{equation}
    \mathcal{L}_{CF} \approx \max(0, m_s - 0) = m_s
\end{equation}
Since $m_s > 0$, the loss is non-zero and the system is not at a stationary point.
The gradient descent update is:
\begin{equation}
    z^{(t+1)} \leftarrow z^{(t)} - \eta \nabla_z \mathcal{L}_{CF}
\end{equation}
The term $\nabla_z \mathcal{L}_{CF}$ acts to maximize the difference $\Delta = V_S(V^+) - V_S(V^-)$. To increase $\Delta$, $z$ must move away from $z_{corr}$ towards a manifold that is sensitive to structural changes (causality) rather than just feature presence.
Therefore, the stable equilibrium exists only where $\|z_{cause} - z_{corr}\| > \xi$, contradicting the collapse hypothesis $H_0$.
\end{proof}

\noindent \textit{Proof of Proposition \ref{prop:main_identifiability}}:
Combining Lemma \ref{lemma:temporal} and Lemma \ref{lemma:semantic}, we conclude that all operators in $Z$ are subject to forces (either hard temporal constraints or soft counterfactual margins) that penalize feature collapse. The codebook $Z$ thus remains identifiable throughout the optimization process.
\qed

\section{Extended Analysis: Robustness of Symbolic Identifiability}
\label{sec:extended_theory}

While Section \ref{sec:identifiability_proof} established the non-collapse property of the operator codebook, we now address four deeper theoretical constraints necessary for rigorous neuro-symbolic grounding: (1) Identifiability up to permutation, (2) Prevention of gradient bypass, (3) Multi-class semantic separation, and (4) Chain-level uniqueness\cite{locatello2019challenging,oord2017neural}.

\subsection{Identifiability Modulo Permutation}
\label{subsec:gap1_permutation}

\textbf{Address to Gap 1:} Since the mapping from discrete symbols to continuous embeddings is learned, we must define identifiability relative to the geometric symmetries of the latent space (e.g., rotation or permutation of basis vectors).

\begin{proposition}[Identifiability up to Isometry]
\label{prop:permutation}
Let $Z^*$ be the ground-truth operator representations. The learned codebook $Z$ is considered identifiable if it recovers $Z^*$ up to an isometric transformation. Specifically, for any learned solution $\hat{Z}$, there exists a permutation matrix $P$ and an orthogonal rotation matrix $Q$ such that:
\begin{equation}
    \hat{Z} \approx P Q Z^*
\end{equation}
\end{proposition}
\begin{proof}
The loss functions $\mathcal{L}_{LTL}$ and $\mathcal{L}_{CF}$ are defined on scalar distinctness (dot products and norms). Since $\|Qz_i - Qz_j\|_2 = \|z_i - z_j\|_2$ for orthogonal $Q$, the objective function is invariant to global rotation. Similarly, since the Generator samples indices $k$, swapping index $i$ and $j$ alongside swapping $z_i$ and $z_j$ (via $P$) leaves the distribution $P(\mathcal{C}|\mathcal{E})$ invariant. Thus, the optimization converges to the unique equivalence class of codebooks defined by the quotient space $Z / \sim_{iso}$, ensuring that the \textit{relative} geometric structure of operators is preserved.
\end{proof}

\subsection{Non-Degenerate Gradient Flow}
\label{subsec:gap2_gradients}

\textbf{Address to Gap 2:} We reject the hypothesis that the optimizer effectively ignores $Z$ by shifting all discrimination burdens to the Eventifier or Verifier weights (a "gradient bypass" attack).

\begin{lemma}[Conditioning Dependency]
\label{lemma:gradient_dependency}
The gradients with respect to the operator embeddings are non-vanishing, i.e., $\nabla_{z} \mathcal{L} \neq \mathbf{0}$, provided the set of logical functions $\mathcal{F}_{logic} = \{f_z : \mathcal{V} \times \mathcal{V} \to [0,1] \mid z \in Z\}$ are functionally distinct.
\end{lemma}
\begin{proof}
The Semantic Verifier is modeled as a conditional probability estimator $V_S(v_a, v_b | z)$.
Assume for contradiction that the gradients $\nabla_z \mathcal{L} \to \mathbf{0}$. This implies the Verifier's output is independent of $z$: $V_S(v_a, v_b | z) \approx g(v_a, v_b)$.
However, distinct operators (e.g., $z_{cause}$ vs. $z_{prevent}$) have contradictory truth values for the same visual evidence pair $(v_a, v_b)$. For a video depicting a "kick" followed by a "goal":
\begin{itemize}
    \item Target for $z_{cause}$ is 1 (True).
    \item Target for $z_{prevent}$ is 0 (False).
\end{itemize}
A $z$-independent function $g(v_a, v_b)$ cannot satisfy both targets simultaneously, leading to high loss. Therefore, to minimize $\mathcal{L}_{total}$, the Verifier \textit{must} utilize the information in $z$, ensuring non-zero gradients that drive $z_{cause}$ and $z_{prevent}$ apart.
\end{proof}

\subsection{Pairwise Separation of Semantic Operators}
\label{subsec:gap3_multi_semantic}

\textbf{Address to Gap 3:} We extend the separation proof from binary cases (cause vs. correlation) to the full multi-class codebook (e.g., cause vs. enable vs. prevent).

\begin{lemma}[Pairwise Multi-Margin Separation]
\label{lemma:multi_semantic}
Under the counterfactual objective, for any triplet of distinct semantic operators $z_a, z_b, z_c \in Z_S$, the learned embeddings satisfy a pairwise separation constraint:
\begin{equation}
    \min_{i,j \in \{a,b,c\}, i \neq j} \|z_i - z_j\|_2 > \delta
\end{equation}
\end{lemma}
\begin{proof}
The Counterfactual Loss $\mathcal{L}_{CF}$ is not limited to binary "real vs. fake" shuffling. It implicitly enforces a multi-way classification margin. For any video sample $V$ valid under operator $z_a$ (ground truth), treating it as a sample for operator $z_b$ yields a "logical negative."
The optimization of $\mathcal{L}_{LTL}$ effectively maximizes the likelihood $P(z_a | V)$ while minimizing $P(z_b | V)$ and $P(z_c | V)$ via the softmax normalization in the generator's policy or the sigmoid saturation in the verifier.
By the Triangle Inequality in the metric space, minimizing the overlap between the decision boundaries of $z_a, z_b$, and $z_c$ necessitates maximizing the Euclidean distance between their conditioning embeddings. Thus, the separation generalizes to all functionally distinct operators in $Z_S$.
\end{proof}

\subsection{Corollary: Identifiability of Reasoning Chains}
\label{subsec:gap4_chain}

\textbf{Address to Gap 4:} Finally, we link the identifiability of individual operators to the uniqueness of the entire reasoning chain $\mathcal{C}$.

\begin{corollary}[Reasoning Chain Identifiability]
\label{corollary:chain_identifiability}
Given an identifiable operator codebook $Z$ (Prop. \ref{prop:permutation}) and a grounded event set $\mathcal{E}$, the mapping from a logical sequence to its vector representation $\Phi(\mathcal{C})$ is injective. Specifically:
\begin{equation}
    \mathcal{C}_i \neq \mathcal{C}_j \implies \Phi(\mathcal{C}_i) \neq \Phi(\mathcal{C}_j)
\end{equation}
\end{corollary}
\begin{proof}
A reasoning chain $\mathcal{C}$ is a sequence of tokens alternating between grounded events $E \in \mathcal{E}$ and logical operators $z \in Z$.
\begin{enumerate}
    \item From Section \ref{sec:identifiability_proof}, distinct operators map to distinct embeddings: $z_m \neq z_n \implies \mathbf{e}(z_m) \neq \mathbf{e}(z_n)$.
    \item From the Eventifier definition (Sec. 3.2), distinct events have distinct spatiotemporal features: $E_p \neq E_q \implies v_p \neq v_q$.
    \item The chain representation $\Phi(\mathcal{C})$ is computed via a Transformer or sequence encoder (D). Under the assumption of Positional Encoding uniqueness and the universal approximation theorem for Transformers, the encoding of a sequence of distinct discrete tokens is unique.
\end{enumerate}
Therefore, since the atomic components ($E$ and $z$) are distinguishable and the composition function is order-sensitive (via positional encodings), distinct reasoning strategies result in distinct latent chains. This guarantees that the model cannot "hide" illogical reasoning behind a collapsed chain representation.
\end{proof}
\section{Discussion: Methodological Paradigm and Complexity Analysis}
\label{sec:discussion_complexity}

To clearly position \textbf{LogicAgent} within the current landscape of video reasoning methodologies, we provide a systematic comparison against dominant paradigms in Table~\ref{tab:paradigm_complexity}. Our framework departs from correlation-driven latent modeling and instead adopts a neuro-symbolic, \emph{explicit} reasoning perspective. This shift fundamentally changes how a model interprets, verifies, and computes over temporal processes, resulting in both stronger interpretability and more efficient computation.

% --- TABLE START ---
\begin{table*}[t]
\centering
\resizebox{\textwidth}{!}{%
\begin{tabular}{l|c|cc|cc}
\toprule
\multirow{2}{*}{\textbf{Methodology}} & \textbf{Paradigm} & \multicolumn{2}{c|}{\textbf{Reasoning Mechanism}} & \multicolumn{2}{c}{\textbf{Computational Profile}} \\
& \textit{(Black-box vs. White-box)} & \textbf{Logical Representation} & \textbf{Verifiability} & \textbf{Input Scale} & \textbf{Complexity} \\
\midrule
\textbf{End-to-End LVLMs} & Implicit & Latent Attention & Low & Dense Tokens ($T$) & $\mathcal{O}(T^2)$ \\
{\small (e.g., VideoLLaVA, GPT-4o)} & (Correlation-driven) & (Continuous Vectors) & (Unverifiable) & (Frame-level) & (High Redundancy) \\
\midrule
\textbf{Dense Captioning} & Implicit & Positional Embeddings & Medium & Frame Features ($T$) & $\mathcal{O}(T^2)$ \\
{\small (e.g., Vid2Seq, CEN)} & (Sequence-based) & (Implicit State) & (Text-only) & (Frame-level) & (Quadratic) \\
\midrule
\rowcolor{gray!10} \textbf{LogicAgent (Ours)} & \textbf{Explicit} & \textbf{Symbolic Chain $\mathcal{C}$} & \textbf{High} & \textbf{Sparse Events ($K$)} & $\mathbf{\mathcal{O}(K^2)}$ \\
\rowcolor{gray!10} & \textbf{(Reasoning-driven)} & \textbf{(Neuro-Symbolic)} & \textbf{(Functional)} & \textbf{($K \ll T$)} & \textbf{(Efficient)} \\
\bottomrule
\end{tabular}%
}
\caption{\textbf{Comparison of reasoning paradigms and complexity.} Unlike baselines that operate on dense frame tokens, LogicAgent lifts video understanding into a symbolic event space. This transition transforms reasoning from an implicit correlation-seeking process into an explicit and verifiable computational graph, while also reducing complexity from $\mathcal{O}(T^2)$ to $\mathcal{O}(K^2)$ since $K \ll T$.}
\label{tab:paradigm_complexity}
\end{table*}
% --- TABLE END ---

\noindent \textbf{White-Box Verifiability vs.~Black-Box Correlation.}
Current LVLMs and dense captioning systems encode the narrative process $\mathcal{P}$ into continuous, opaque latent vectors or global attention maps. Although these methods excel at perceptual alignment, their ``black-box'' representations blur causal structure with spurious statistical co-occurrence. This often results in narrative fragmentation, reversed causality, or visually unsupported assertions.

LogicAgent introduces an explicit \textbf{Reasoning Bottleneck} by requiring the model to produce a discrete symbolic chain~$\mathcal{C}$. Each step in the chain corresponds to a verifiable operator (\textit{before}, \textit{cause}, \textit{enable}, etc.), allowing our Hybrid Differentiable Logic Verifier $\mathcal{F}_V$ to evaluate consistency against video evidence. This design exposes the internal reasoning process, enabling the system to distinguish whether an error arises from perceptual grounding or from faulty logical inference—capabilities unavailable in end-to-end LVLMs.

\vspace{2mm}
\noindent \textbf{Complexity and Sparse Efficiency.}
Modular reasoning systems are sometimes assumed to incur additional overhead; however, LogicAgent is computationally \emph{more} efficient than dense LVLM pipelines due to \textbf{Semantic Sparsity}. 

\begin{itemize}
    \item \textbf{Baseline LVLMs.} Transformer-based models operate on a dense sequence of frame-level tokens of length $T$, invoking quadratic self-attention $\mathcal{O}(T^2)$. This becomes prohibitive for long-horizon videos such as Ego4D or cinematic content.
    \item \textbf{LogicAgent.} Our Eventifier compresses raw video frames into a set of abstract event units $K$, where $K \ll T$. A typical 1{,}000-frame clip contains only $K\approx10$ distinct events. The symbolic CoT Generator then reasons solely over this compact structure.
\end{itemize}

Thus, the core reasoning complexity reduces from $\mathcal{O}(T^2)$ to $\mathcal{O}(K^2)$, yielding substantial savings. Moreover, our sparsity objective $\mathcal{L}_{\text{spar}}$ penalizes redundant steps, producing concise chains (average length 5.4), which further accelerates reasoning and improves robustness.

% --- SECOND TABLE ---
\begin{table*}[t]
\centering
\small
\resizebox{0.95\textwidth}{!}{%
\begin{tabular}{l|l|l}
\toprule
\textbf{Framework Type} & \textbf{Primary Training Objective} & \textbf{Theoretical Limitation} \\
\midrule
\textbf{Standard Captioning} 
& $\min_{\theta} -\sum \log P(Y_t \mid V, Y_{<t})$ 
& \textbf{Process Blindness:} Ignores \emph{why} events occur; overfits to linguistic bias. \\
(e.g., GIT, Vid2Seq) 
& (Sequence-to-Sequence Likelihood) 
& No structural grounding or causal constraint. \\
\midrule
\textbf{Latent Reasoning} 
& $\min_{\theta} -\sum \log P(Y_t \mid V_{\text{enc}}, Y_{<t})$
& \textbf{Unverifiable:} Reasoning stays implicit in $V_{\text{enc}}$. \\
(e.g., VideoLLaVA) 
& (End-to-End Latent Alignment) 
& Cannot diagnose or constrain perception vs.~logic errors. \\
\midrule
\rowcolor{gray!10} \textbf{LogicAgent (Ours)} 
& $\min_{\theta} \mathcal{L}_{\text{caption}} 
 + \lambda(\mathcal{L}_{\text{LTL}} 
 + \mathcal{L}_{\text{pred}} 
 + \mathcal{L}_{\text{CF}} 
 + \mathcal{L}_{\text{spar}})$ 
& \textbf{None:} Logical consistency, predictive utility, and causal robustness enforced through explicit chain $\mathcal{C}$. \\
\rowcolor{gray!10} \textbf{(Neuro-Symbolic)}
& \textbf{Structure-Constrained Functional Optimization}  
& Provides verifiable, interpretable, and controllable reasoning. \\
\bottomrule
\end{tabular}%
}
\caption{\textbf{Comparison of training objectives.} Unlike prior paradigms that optimize purely for text reconstruction, LogicAgent introduces a functional optimization view of reasoning. Auxiliary logical objectives explicitly shape the latent space into one that encodes causal, temporal, and counterfactual structure.}
\label{tab:objective_comparison}
\end{table*}
% --- END SECOND TABLE ---

In summary, LogicAgent transforms reasoning from a passive byproduct of latent correlation into an explicit, verifiable computational step. This methodological shift enables a model to internalize \emph{why} events unfold in a narrative---and to do so efficiently, transparently, and with strong empirical performance across diverse video understanding benchmarks.

\section{Additional Technical Details: Verifier, Sampling, Counterfactuals, Complexity, and Case Analysis}

This section provides additional details that complement the main paper. We unify the derivations and methodological explanations into a single section to avoid formatting overhead.

\subsection{Differentiable Hybrid Logic Verifier}

% 确保导言区有:
% \usepackage{booktabs}
% \usepackage[table]{xcolor}
% \definecolor{decentblue}{RGB}{235, 243, 255}

% =========================================================================
% Table: Computational Breakdown (Efficiency)
% =========================================================================
\begin{table*}[h]
\centering
\small
\setlength{\tabcolsep}{8pt} % 增加列间距，更宽敞
\renewcommand{\arraystretch}{1.2}
\caption{\textbf{Computational breakdown of LogicAgent.} The symbolic reasoning pipeline (Eventifier + CoT Generator + Verifier) is extremely lightweight, contributing \textbf{$<$3\%} overhead to FLOPs and Latency.}
\label{tab:efficiency}

\begin{tabular}{l c c c c c c}
\toprule
\textbf{Component} & \textbf{FLOPs} & \textbf{Params} & \textbf{Event Cost} & \textbf{Chain Cost} & \textbf{Overhead} & \textbf{Latency} \\
\midrule

% Backbone (Heavy)
Video Backbone & 165.2G & 1.43B & -- & -- & -- & 148 ms \\
\midrule

% LogicAgent Components (Light)
Eventifier $F_E$            & 2.31G & 38M   & $\mathcal{O}(K)$ & -- & -- & 7 ms \\
Discrete CoT Generator $F_G$& 1.82G & 26M   & -- & $\mathcal{O}(T)$ & -- & 6 ms \\
Hybrid Logic Verifier $F_V$ & 0.96G & 12M   & $\mathcal{O}(K)$ & $\mathcal{O}(T)$ & $<3\%$ & 5 ms \\
Narrative Decoder $D$       & 12.4G & 180M  & -- & -- & -- & 18 ms \\
\midrule

% Total (Highlighted)
\rowcolor{decentblue}
\textbf{LogicAgent (Total)} & \textbf{182.7G} & \textbf{1.69B} & \textbf{$\mathcal{O}(K)$} & \textbf{$\mathcal{O}(T)$} & \textbf{$<3\%$} & \textbf{184 ms} \\
\bottomrule
\end{tabular}
\end{table*}

% =========================================================================
% Table: Hyperparameter Sensitivity
% =========================================================================
\begin{table*}[h]
\centering
\small
\setlength{\tabcolsep}{6pt}
\renewcommand{\arraystretch}{1.15}
\caption{\textbf{Hyperparameter sensitivity analysis.} Performance is robust across variations in Event Count $K$, Chain Length $T$, and Codebook Size $M$. \colorbox{decentblue}{Highlighted rows} denote the default settings used in our main experiments.}
\label{tab:hyper_sensitivity}

\begin{tabular}{l c c c c c c}
\toprule
\textbf{Hyperparameter} & \textbf{Setting} & \textbf{VIST} & \textbf{Ego4D} & \textbf{MMIU} & \textbf{PororoSV} & \textbf{WebQA} \\
\midrule

% Group 1: Event Count K
\multirow{5}{*}{\textbf{Event Count $K$}}
  & 4  & 0.4521 & 0.4743 & 0.3011 & 0.4452 & 0.6185 \\
  & 6  & 0.4553 & 0.4768 & 0.3024 & 0.4476 & 0.6201 \\
  \rowcolor{decentblue}
  & \textbf{8}  & \textbf{0.4564} & \textbf{0.4797} & \textbf{0.3063} & \textbf{0.4497} & \textbf{0.6228} \\
  & 12 & 0.4535 & 0.4751 & 0.3034 & 0.4481 & 0.6196 \\
  & 16 & 0.4498 & 0.4719 & 0.3002 & 0.4445 & 0.6167 \\
\midrule

% Group 2: Chain Length T
\multirow{5}{*}{\textbf{Chain Length $T$}}
  & 3  & 0.4541 & 0.4754 & 0.3028 & 0.4471 & 0.6194 \\
  & 4  & 0.4556 & 0.4773 & 0.3047 & 0.4483 & 0.6210 \\
  \rowcolor{decentblue}
  & \textbf{5}  & \textbf{0.4564} & \textbf{0.4797} & \textbf{0.3063} & \textbf{0.4497} & \textbf{0.6228} \\
  & 6  & 0.4551 & 0.4762 & 0.3049 & 0.4481 & 0.6213 \\
  & 8  & 0.4530 & 0.4747 & 0.3033 & 0.4468 & 0.6189 \\
\midrule

% Group 3: Codebook Size M
\multirow{4}{*}{\textbf{Codebook Size $M$}}
  & 8   & 0.4543 & 0.4759 & 0.3035 & 0.4479 & 0.6204 \\
  & 16  & 0.4558 & 0.4781 & 0.3048 & 0.4486 & 0.6218 \\
  \rowcolor{decentblue}
  & \textbf{32}  & \textbf{0.4564} & \textbf{0.4797} & \textbf{0.3063} & \textbf{0.4497} & \textbf{0.6228} \\
  & 64  & 0.4552 & 0.4775 & 0.3041 & 0.4482 & 0.6211 \\
\bottomrule
\end{tabular}
\end{table*}

Each chain element $(a_i, z_i, b_i)$ compares two grounded events $E_{a_i}=(v_{a_i},\tau_{a_i})$ and $E_{b_i}=(v_{b_i},\tau_{b_i})$. The verifier score is:
\[
s_{\text{LTL}}(\mathcal{C},\mathcal{E})=\prod_{i=1}^{|\mathcal{C}|} s_i,
\quad \mathcal{L}_{\text{LTL}}=-\sum_i \log s_i.
\]

\textbf{Temporal operator (e.g., before).}
\[
s_i^{\text{temp}} = \sigma\!\left(k(\tau_{b_i}-\tau_{a_i})\right),
\]
\[
\frac{\partial s_i^{\text{temp}}}{\partial \tau_{b_i}} 
= k s_i(1-s_i),\quad 
\frac{\partial s_i^{\text{temp}}}{\partial \tau_{a_i}} 
= -k s_i(1-s_i),
\]
providing a stable monotonic surrogate that pushes timestamps toward the correct temporal order.

\textbf{Semantic operator (e.g., cause).}
\[
s_i^{\text{sem}}=\sigma(f_\phi(v_{a_i},v_{b_i},\mathrm{emb}(z_i))),
\]
where $f_\phi$ is an MLP and $\mathrm{emb}(z_i)$ is from the operator codebook. Gradients propagate to both event features and operator embeddings:
\[
\frac{\partial s_i^{\text{sem}}}{\partial \mathrm{emb}(z_i)}
=\sigma'(h_i)\frac{\partial f_\phi}{\partial \mathrm{emb}(z_i)}.
\]

Thus, all operators receive direct, non-degenerate gradients.

\subsection{Training the Discrete Chain-of-Thought Generator}

The chain generator defines a stochastic policy:
\[
\pi_\theta(\mathcal{C}\mid\mathcal{E})
=\prod_t \pi_\theta(a_t,z_t,b_t \mid \mathcal{C}_{<t},\mathcal{E}).
\]

We optimize the functional objective:
\[
\nabla_\theta \mathbb{E}_{\mathcal{C}\sim\pi_\theta}[\mathcal{L}_{\text{aux}}]
=\mathbb{E}\!\left[(\mathcal{L}_{\text{aux}}-b)\,\nabla_\theta\log \pi_\theta(\mathcal{C})\right].
\]

To improve stability, operator selection uses a Gumbel-Softmax relaxation during training:
\[
\tilde{p}_j=\frac{\exp((\log p_j+g_j)/\tau)}{\sum_{j'}\exp((\log p_{j'}+g_{j'})/\tau)}.
\]
At inference we use hard sampling and then perform $K$-sample self-consistency to select the best chain.

The sparsity loss $\mathcal{L}_{\text{spar}}$ penalizes expected chain length, yielding compact chains (avg.\ 5.4 steps).

\subsection{Counterfactual Construction}

We construct $V^-$ by altering structure but preserving visual appearance.

\textbf{1. Temporal perturbation.}  
Given intervals $I_a,I_b$ of events forming a causal pair $E_a \rightarrow E_b$, we shuffle or swap their segments:
\[
V^-=\Psi_{\text{time}}(V),\quad \Psi_{\text{time}}:\ \text{break only the target relation}.
\]
This keeps the video visually similar while destroying the supporting temporal evidence.

\textbf{2. Cross-video structural negatives.}  
Retrieve $V'$ satisfying:
\[
\mathrm{sim}_{\text{vis}}(V,V')\ge \delta_{\text{vis}},\quad 
\mathrm{sim}_{\text{logic}}(\mathcal{C}(V),\mathcal{C}(V'))\le \delta_{\text{logic}}.
\]
Thus $V'$ looks similar to $V$ but expresses a different causal narrative.

\textbf{3. Margin constraint.}  
\[
\mathcal{L}_{\text{CF}}=\max(0,m_s-(s_{\text{LTL}}(V)-s_{\text{LTL}}(V^-))),
\]
sharpening the verifier’s ability to distinguish real vs.\ counterfactual structure.

\subsection{Complexity Analysis and Theoretical Statement}

Let $T$ be frame tokens and $K$ the event tokens ($K\ll T$). Dense models reason with $\mathcal{O}(T^2)$ attention.

LogicAgent lifts reasoning into event space. The verifier and chain generator operate only on $K$ events:
\[
\mathcal{O}(K^2)\quad\text{vs.}\quad \mathcal{O}(T^2).
\]

\textbf{Theorem.}  
Assuming the Eventifier produces $K$ events with bounded overlap, the reasoning cost of LogicAgent is upper bounded by:
\[
\mathcal{O}(K^2 + |\mathcal{C}|d),
\]
where $|\mathcal{C}|$ is chain length and $d$ is embedding dimension. Since $|\mathcal{C}|$ is constrained by $\mathcal{L}_{\text{spar}}$ and $K\ll T$, LogicAgent is asymptotically cheaper than dense-transformer reasoning.

\textit{Proof sketch.}
Temporal and semantic verification reduce to comparing event pairs $(E_a,E_b)$, at most $K^2$ combinations. Chain evaluation adds a linear term in $|\mathcal{C}|$. No term scales with raw frames.
\hfill $\square$

\subsection{Failure Case Analysis}

We visualize three common failure types:

\textbf{(1) Perception failure.}  
Eventifier mis-detects event boundaries (e.g., merging two actions). The verifier then rejects many candidate chains, producing a short but incorrect chain. Visualization highlights missing event slots.

\textbf{(2) Temporal ambiguity.}  
Rapid or simultaneous motions cause the temporal surrogate $\sigma(k(\tau_b-\tau_a))$ to yield intermediate scores. Chains often flip \textit{before} / \textit{after}. Highlighted timelines show two events with nearly identical timestamps.

\textbf{(3) Semantic confusion.}  
Similar object interactions lead to misclassification of \textit{cause} vs.\ \textit{enable}. Verifier heatmaps typically show low semantic confidence despite correct temporal order.

In all cases, the symbolic chain makes the failure mode explicit: the system reveals \emph{where} reasoning broke (event grounding vs.\ temporal order vs.\ causal semantics), enabling transparent diagnosis—unlike end-to-end LVLMs whose internal reasoning is opaque.

\section{Selection and Efficiency Analysis}
\label{sec:efficiency}

To quantify the computational footprint of each symbolic component in LogicAgent,
we report FLOPs, parameter counts, event-level and chain-level complexity, verifier overhead,
and end-to-end inference latency. Table~\ref{tab:efficiency} summarizes the results.

\noindent\textbf{Selection Justification.}
All symbolic components are designed with bounded and interpretable complexity.
The Eventifier $F_E$ processes only the $K$ extracted event candidates and applies 
one classification and pooling stage per event, yielding a strictly linear 
$\mathcal{O}(K)$ cost independent of video length. The Discrete CoT Generator $F_G$ 
samples a symbolic chain of length $T$, where each step depends only on the 
previous symbolic token and the operator codebook, resulting in an 
$\mathcal{O}(T)$ symbolic-generation cost.

The Hybrid Differentiable Logic Verifier $F_V$ performs verification solely 
at the symbolic level. Temporal operators from $Z_T$ require only constant-time 
timestamp comparisons (e.g., $\sigma(k(\tau_b-\tau_a))$), while semantic operators 
from $Z_S$ use a shallow MLP over event-level visual features $(v_b, v_c)$.
Thus, the verifier scales only with the number of operator--event pairs 
($K$, $T$), never with dense video tokens.

Empirically, the entire symbolic reasoning pipeline contributes 
\textbf{less than 3\% of total FLOPs} and adds only a few milliseconds to inference.
The computational cost is dominated by the video backbone, confirming that 
LogicAgent introduces explicit logical structure and verifiable temporal--causal 
consistency \emph{without} compromising efficiency.

\section{Hyperparameter Sensitivity Analysis}
\label{sec:hyper_sensitivity_section}

To verify that LogicAgent does not rely on carefully tuned hyperparameters, 
we conduct a comprehensive sensitivity study on the three symbolic hyperparameters:
the number of extracted events $K$, the reasoning chain length $T$, and the operator
codebook size $M$. Table~\ref{tab:hyper_sensitivity} reports the results across all
benchmarks (VIST, Ego4D, MMIU, PororoSV, WebQA) using BLEURT.

\noindent\textbf{Event Count $K$.}
We vary $K$ from 4 to 16. BLEURT varies within \textbf{0.7}, and the optimal region
is broad. Since the Eventifier grounds events by pooling rather than rigid segmentation,
the model remains stable across a wide range of event granularities.

\noindent\textbf{Chain Length $T$.}
Chain depths from 3 to 8 show less than \textbf{0.5} variance. 
The sparsity regularizer $L_{\text{spar}}$ automatically collapses redundant steps, 
preventing overfitting to long reasoning chains and ensuring consistent performance.

\noindent\textbf{Codebook Size $M$.}
Increasing $M$ from 8 to 64 shows minimal variation ($<0.4$ BLEURT). 
The hybrid logic verifier $F_V$ naturally suppresses unused operators, making the model
robust even with large symbolic vocabularies.

\noindent\textbf{Conclusion.}
Across all hyperparameter sweeps, LogicAgent exhibits 
\textbf{exceptional stability} and \textbf{does not rely on hyperparameter tuning}.
This robustness stems from its grounded event abstraction, sparsity-driven reasoning,
and verifier-based implicit regularization.

\section{Datasets}
\label{sec:datasets}
\paragraph{VIST (Visual Storytelling\cite{VIST}).}
VIST is a benchmark for multi-image narrative generation, where each sample consists of five semantically related images paired with human-written stories. It emphasizes cross-frame entity tracking, temporal coherence, and narrative consistency, making it suitable for evaluating high-level story reasoning and multi-sentence visual understanding.

\paragraph{Ego4D\cite{Ego4D}.}
Ego4D is a large-scale egocentric video dataset capturing long-horizon sequences of daily activities. Its first-person perspective presents rich temporal dependencies, fine-grained interactions, and complex action transitions, offering a challenging test bed for temporal ordering, causal reasoning, and process-level video understanding.

\paragraph{MMIU (Multimodal Multi-Image Understanding\cite{MMIU}).}
MMIU focuses on reasoning across multiple images, requiring cross-image comparison, multi-entity relation modeling, and multimodal aggregation. It evaluates a model’s ability to integrate evidence from heterogeneous visual sources and perform structured logical inference in multi-image scenarios.

\paragraph{PororoSV (Story Visualization\cite{PororoSV}}
PororoSV contains short cartoon story clips with clear character dynamics and causal transitions. The dataset is ideal for assessing story-level causal reasoning, temporal progression modeling, and event-driven logical inference.

\paragraph{WebQA.\cite{WebQA}}
WebQA is an open-domain multimodal QA dataset with real-world web images paired with natural-language questions. It challenges models to ground visual evidence, perform cross-modal reasoning, and integrate external knowledge, making it a strong benchmark for open-vocabulary visual-semantic inference.

\paragraph{MSR-VTT.\cite{Msr-vtt}}
MSR-VTT is a widely used video-to-text dataset covering diverse daily scenarios. Each video has multiple human-authored descriptions, enabling evaluation of semantic correctness, linguistic fluency, visual grounding, and narrative coherence. It is a standard benchmark for comprehensive video captioning evaluation.

\paragraph{YouCook2.\cite{YouCook2}}
YouCook2 contains long-horizon cooking videos segmented into fine-grained procedural steps, each annotated with a sentence-level description. Its structured workflow makes it suitable for evaluating event segmentation, procedural reasoning, and multi-step video understanding.

\paragraph{ViTT.\cite{ViTT}}
ViTT provides densely annotated video segments with precise temporal boundaries and detailed textual descriptions. It stresses accurate temporal localization, event decomposition, and multi-step narrative consistency under strict temporal constraints.

\paragraph{ActivityNet Captions.\cite{ActivityNet}}
ActivityNet Captions provides long-form activity videos with both segment-level and paragraph-level annotations. Its diverse event types and temporal spans make it ideal for evaluating long-range dependencies, semantic diversity, and multi-sentence narrative coherence.

\subsection{Baseline Models}

\paragraph{VideoLLaVA.\cite{Video-LLaVA}}
VideoLLaVA extends the LLaVA framework to the video domain by integrating multi-frame visual features into an LLM backbone. Built upon CLIP-based encoders and simple frame-level feature concatenation, it achieves strong general-purpose video understanding but lacks structured temporal modeling and explicit event-level reasoning capabilities.

\paragraph{ShareGPT4Video.\cite{ShareGPT4Video}}
ShareGPT4Video enhances video understanding by leveraging high-quality captions as supervisory signals, improving video--language alignment through large-scale, well-annotated descriptions. While it produces detailed outputs, the improvements mainly arise from linguistic enrichment rather than explicit temporal or causal modeling.

\paragraph{Qwen2.5-VL-7B.\cite{Qwen2.5-VL}
}
Qwen2.5-VL-7B is a powerful vision--language model capable of open-domain video QA, scene understanding, and multimodal reasoning. Despite its strong generalization ability, its reasoning remains implicit and correlation-driven, often resulting in temporal confusion or causal inconsistency in multi-event video scenarios.

% =========================================================================
% Table 8: ViTT Transfer (修复版：去掉 resizebox，字体正常)
% =========================================================================
\begin{table}[t]
\centering
\small
% 适当增加列间距，让表格自然撑开一点，但不改变字号
\setlength{\tabcolsep}{5pt} 
\renewcommand{\arraystretch}{1.2}
\caption{\textbf{Cross-task reasoning transfer evaluation on ViTT (10\% data).} We compare models trained from scratch vs. pretrained on YouCook2. LogicAgent shows superior transferability.}
\label{tab:transfer_vitt}

% 注意：这里删除了 \resizebox{\columnwidth}{!}{...}
\begin{tabular}{l c c c c c c}
\toprule
\multirow{2.5}{*}{\textbf{Model}} & \multirow{2.5}{*}{\textbf{Pretrain}} & \multicolumn{3}{c}{\textbf{Why/How Questions}} & \multicolumn{2}{c}{\textbf{All Questions}} \\
\cmidrule(lr){3-5} \cmidrule(lr){6-7}
& & Acc & WUPS & CIDEr & Acc & F1 \\
\midrule

\multirow{2}{*}{Vid2Seq} 
& \textit{None}      & 47.89 & 0.352 & 15.73 & 46.42 & 33.05 \\
& \textit{YouCook2}  & 49.34 & 0.370 & 16.92 & 47.65 & 34.28 \\
\midrule

\multirow{2}{*}{GIT} 
& \textit{None}      & 51.24 & 0.374 & 17.32 & 50.15 & 35.68 \\
& \textit{YouCook2}  & 53.67 & 0.396 & 18.94 & 51.83 & 37.21 \\
\midrule

\multirow{2}{*}{CEN} 
& \textit{None}      & 54.08 & 0.403 & 19.27 & 52.73 & 38.14 \\
& \textit{YouCook2}  & 57.26 & 0.432 & 21.54 & 54.85 & 40.27 \\
\midrule

% LogicAgent Block
\rowcolor{decentblue}
 & \textit{None}      & 56.73 & 0.425 & 20.58 & 54.42 & 39.36 \\
\rowcolor{decentblue}
\multirow{-2}{*}{\textbf{LogicAgent}}
& \textit{YouCook2}  & \textbf{63.15} & \textbf{0.482} & \textbf{25.36} & \textbf{59.37} & \textbf{44.28} \\

\bottomrule
\end{tabular}
\end{table}

% =========================================================================
% Table 9: ActivityNet Transfer (修复版)
% =========================================================================
\begin{table}[t]
\centering
\small
\setlength{\tabcolsep}{5pt}
\renewcommand{\arraystretch}{1.2}
\caption{\textbf{Cross-task reasoning transfer evaluation on ActivityNet (10\% data).} Explicit reasoning structure in LogicAgent enables effective knowledge transfer from YouCook2.}
\label{tab:transfer_activitynet}

% 注意：这里也删除了 \resizebox
\begin{tabular}{l c c c c c c}
\toprule
\multirow{2.5}{*}{\textbf{Model}} & \multirow{2.5}{*}{\textbf{Pretrain}} & \multicolumn{3}{c}{\textbf{Why/How Questions}} & \multicolumn{2}{c}{\textbf{All Questions}} \\
\cmidrule(lr){3-5} \cmidrule(lr){6-7}
& & Acc & WUPS & CIDEr & Acc & F1 \\
\midrule

\multirow{2}{*}{Vid2Seq} 
& \textit{None}      & 45.18 & 0.332 & 14.53 & 44.26 & 31.08 \\
& \textit{YouCook2}  & 48.54 & 0.361 & 16.27 & 46.72 & 32.93 \\
\midrule

\multirow{2}{*}{GIT} 
& \textit{None}      & 48.36 & 0.358 & 16.25 & 47.13 & 33.54 \\
& \textit{YouCook2}  & 52.42 & 0.387 & 18.36 & 49.87 & 35.62 \\
\midrule

\multirow{2}{*}{CEN} 
& \textit{None}      & 50.65 & 0.381 & 17.62 & 49.27 & 35.14 \\
& \textit{YouCook2}  & 55.92 & 0.419 & 20.84 & 53.18 & 38.75 \\
\midrule

% LogicAgent Block
\rowcolor{decentblue}
 & \textit{None}      & 52.89 & 0.394 & 18.47 & 50.63 & 36.28 \\
\rowcolor{decentblue}
\multirow{-2}{*}{\textbf{LogicAgent}}
& \textit{YouCook2}  & \textbf{62.76} & \textbf{0.476} & \textbf{24.93} & \textbf{58.42} & \textbf{43.85} \\

\bottomrule
\end{tabular}
\end{table}

\paragraph{GPT-4o.\cite{GPT-4o}}
GPT-4o is one of the strongest commercial multimodal models, demonstrating exceptional visual understanding and open-world reasoning capacity. However, its internal reasoning process is opaque and unverifiable, and can exhibit correlation-induced hallucinations when processing long or causally complex videos.

\paragraph{Gemini 2.5 Pro.\cite{Gemini}}
Gemini 2.5 Pro supports long-context modeling and advanced multimodal fusion, achieving impressive performance in fine-grained recognition and multi-step reasoning. Nevertheless, its reasoning remains entirely latent, lacking explicit event structures or logically verifiable chains.

\paragraph{SEM-POS.\cite{SEM-POS}}
SEM-POS is a structured video captioning approach that enforces grammatical and semantic correctness using predefined templates. Although it generates well-formed descriptions, it focuses on linguistic structure rather than modeling causal or temporal dependencies among events.

\paragraph{Vid2Seq.\cite{Vid2Seq}}
Vid2Seq is a large-scale pretrained video–language model designed for dense video captioning. It generates stepwise descriptions from continuous frame-level features and covers events well, but relies on implicit temporal states without explicit logical modeling of event transitions.

\paragraph{AKGNN.\cite{AKGNN}}
AKGNN employs graph neural networks and an action knowledge graph to enhance semantic alignment between visual actions and linguistic concepts. While it introduces structured prior knowledge, it still lacks explicit causal or temporal reasoning mechanisms.

\paragraph{GIT.\cite{GIT}}
GIT is a unified image-to-text transformer pretrained at scale. When applied to video, it concatenates frame-level embeddings and produces holistic descriptions, but lacks event-level structuring and often yields fragmented narratives in multi-event settings.

\paragraph{CEN (Causal-Enhanced Narration).\cite{CEN}}
CEN incorporates causal–temporal narrative patterns into video captioning to improve consistency across events. While it captures some ordering relations, its causal modeling occurs only at the output stage without explicit, verifiable reasoning structures, limiting robustness in complex multi-event scenarios.

\section{Analyses of Cross-Dataset Transfer Learning}
\label{sec:cross_dataset_transfer}

To assess whether the logical reasoning learned by \textbf{LogicAgent} can transfer to new video understanding tasks, we adopt a two-stage training protocol. 
First, the model is pretrained on the \textbf{YouCook2} dataset to acquire general procedural knowledge. 
Next, we fine-tune the model using \textbf{10\%} of the training data from the target datasets, \textbf{ViTT} and \textbf{ActivityNet}, and evaluate it on the corresponding test splits. 
Our evaluation emphasizes the reasoning accuracy on \emph{``Why/How''} questions, which require causal and procedural inference, as well as the overall performance across all question types. 
For a systematic comparison under identical settings, we further evaluate three representative reasoning-focused models—\textbf{Vid2Seq}, \textbf{GIT}, and \textbf{CEN}.

The results in Tables \ref{tab:transfer_vitt} and \ref{tab:transfer_activitynet} demonstrate that, under the 10\% low-resource setting on ViTT and ActivityNet, all models benefit to some extent from YouCook2 pretraining; however, \textbf{LogicAgent} achieves the most substantial and consistent improvements. 

On the most challenging \emph{``Why/How''} reasoning tasks, LogicAgent gains 
\textbf{+6.42\% Accuracy}, \textbf{+0.057 WUPS}, and \textbf{+3.78 CIDEr} on \textbf{ViTT}, 
and \textbf{+9.87\% Accuracy}, \textbf{+0.082 WUPS}, and \textbf{+6.46 CIDEr} on \textbf{ActivityNet}, 
significantly outperforming Vid2Seq, GIT, and CEN. 
These results indicate that LogicAgent not only acquires procedural knowledge from YouCook2 but is also able to effectively \emph{transfer} such knowledge to unseen domains and leverage it for causal and procedural inference.

Furthermore, LogicAgent exhibits the highest improvements across \emph{all question types}. 
It reaches \textbf{59.37\% Accuracy / 44.28 F1} on ViTT and \textbf{58.42\% Accuracy / 43.85 F1} on ActivityNet, consistently surpassing all baselines. 
In contrast, alternative models show only marginal gains—typically 1--2 points in overall Accuracy—suggesting that their improvements largely stem from superficial linguistic enrichment rather than genuine reasoning transfer.

Notably, the performance gains of LogicAgent are particularly pronounced on reasoning-intensive questions, revealing that the model transfers \emph{process structures} and \emph{causal abstractions} rather than memorizing dataset-specific patterns. 
This consistent trend across datasets confirms that LogicAgent achieves \textbf{strong cross-task logical generalization}, especially in scenarios requiring complex causal or procedural reasoning.

\section{Cross-Domain Reasoning Chain Sensitivity Analysis}
\label{sec:cross_domain_sensitivity}

To rigorously dissect the reasoning mechanism of \textbf{LogicAgent}, we conduct fine-grained parameter sensitivity analyses and operator ablation studies across four cross-domain datasets: \textbf{VIST}, \textbf{Ego4D}, \textbf{MMIU}, and \textbf{WebQA}. All experiments adopt \textbf{BLEURT} as the unified evaluation metric to ensure horizontally comparable semantic consistency.

For the sensitivity analysis, we expand the sparsity regularization coefficient $\alpha$ to a broad range $\{0, 0.01, 0.1, 0.5, 1.0, 5.0\}$, enabling a complete characterization of the performance response curve. This allows us to observe how the average reasoning-chain length (\emph{Avg.\ Chain Length}) transitions from \textit{redundant}, to \textit{optimal}, and eventually to \textit{over-pruned} regimes. In parallel, we further disable either the temporal operator ($w/o~Z_T$) or the semantic operator ($w/o~Z_S$) to quantify the relative contribution of different logical dimensions across diverse tasks.

\begin{table}[h]
\centering
\small
\setlength{\tabcolsep}{8pt}
\renewcommand{\arraystretch}{1.22}

\caption{Cross-domain sparsity sensitivity and operator contribution analysis. BLEURT is used as the unified evaluation metric.}
\label{tab:logicagent_sensitivity}

\resizebox{0.48\textwidth}{!}{
\begin{tabular}{l c c c c c}
\toprule
\textbf{Experiment Setting} & \textbf{Avg. Steps} & \textbf{VIST} & \textbf{MMIU} & \textbf{Ego4D} & \textbf{WebQA} \\
\midrule

\rowcolor{cyan!12}
\multicolumn{6}{l}{\textbf{I. Sparsity Sensitivity ($\alpha$)}} \\

$\alpha = 0$ (Unconstrained)       & 12.4 & 0.4425 & 0.2912 & 0.465 & 0.611 \\
$\alpha = 0.01$ (Light Penalty)    & 8.2  & 0.451  & 0.301  & 0.475 & 0.619 \\

\rowcolor{yellow!10}
$\boldsymbol{\alpha = 0.1}$ (Optimal)      
& \textbf{5.4} & \textbf{0.4564} & \textbf{0.3063} & \textbf{0.4797} & \textbf{0.6228} \\

$\alpha = 0.5$ (Moderate Penalty)  & 3.1 & 0.448  & 0.295  & 0.468 & 0.612 \\
$\alpha = 1.0$ (High Penalty)      & 1.8 & 0.4305 & 0.275  & 0.451 & 0.598 \\
$\alpha = 5.0$ (Extreme Penalty)   & 1.1 & 0.4012 & 0.240  & 0.410 & 0.550 \\

\midrule

\rowcolor{cyan!12}
\multicolumn{6}{l}{\textbf{II. Operator Analysis ($Z\!Z$)}} \\
\rowcolor{yellow!10}
Full Codebook ($Z_T \cup Z_S$)      
& 5.4 & \textbf{0.4564} & \textbf{0.3063} & \textbf{0.4797} & \textbf{0.6228} \\

w/o $Z_S$ (Only Temporal) 
& 5.1 & 0.439 & 0.298 & 0.445 & 0.601 \\

w/o $Z_T$ (Only Semantic) 
& 5.6 & 0.448 & 0.281 & 0.468 & 0.615 \\

\bottomrule
\end{tabular}
}
\end{table}

Table \ref{tab:logicagent_sensitivity} reveals a clear inverted--U trend in model performance as the sparsity
coefficient varies. The setting $\alpha = 0.1$ emerges as the optimal balance
point: a reasoning chain of approximately $5.4$ steps achieves peak performance
across all datasets.

When $\alpha$ is small (0 or 0.01), performance exceeds that of extreme
regularization yet remains sub-optimal, as overly long chains (e.g., 12.4
steps) introduce logical noise that prevents further gains. Conversely, when
$\alpha$ increases to 1.0 or even 5.0, the chain becomes over-compressed to
only 1--2 steps, causing the model to collapse into shallow mappings. This
leads to severe degradation on tasks requiring deep reasoning; for example,
\textbf{Ego4D} drops to 0.41 when $\alpha = 5.0$.

At the operator level, cross-domain comparisons highlight the complementary
value of different logic modules. Removing the temporal operator ($w/o~Z_T$)
causes a sharp performance decline on \textbf{MMIU} ($-8.2\%$), confirming the
decisive role of temporal constraints in long-video understanding. In
contrast, removing the semantic operator ($w/o~Z_S$) yields the largest drop on
\textbf{Ego4D} ($-7.2\%$), indicating that handling complex multimodal relations
relies heavily on semantic-driven causal verification.

% =========================================================================
% Appendix E: Detailed Evaluation Protocols
% =========================================================================
\section{Detailed Evaluation Protocols on MSR-VTT}
\label{sec:metric_details}

To ensure transparency and reproducibility, we provide the rigorous definitions and calculation implementations for the four fine-grained metrics used on MSR-VTT.

\subsection{LLM-based Metric Implementation}
Following recent advancements in verifiable generation evaluation, we employ \textbf{GPT-4o} as an automated judge. Unlike n-gram metrics, this model-based evaluation captures logical contradictions and causal reversals. The LLM scores each metric on a Likert scale of 1-5 based on ground-truth video events.

\textbf{Metric Definitions:}
\begin{itemize}
    \item \textbf{Image Consistency (IC):} Fidelity to visual evidence (Hallucination check).
    \item \textbf{Detail Orientation (DO):} Granularity of fine-grained attributes and motions.
    \item \textbf{Causal Understanding (CU):} Correctness of temporal order and cause-effect links.
    \item \textbf{Coherence (CO):} Linguistic fluency and narrative flow.
\end{itemize}

\subsection{Correlation with Human Judgment}
To validate this protocol, we sampled 100 videos from the MSR-VTT test set. Three expert annotators blindly graded the model outputs. Table~\ref{tab:human_corr} reports the Pearson correlation coefficient ($r$) between GPT-4o scores and average human scores.

\begin{table}[h]
\centering
\small
\setlength{\tabcolsep}{10pt}
\renewcommand{\arraystretch}{1.2}
\begin{tabular}{l c c}
\toprule
\textbf{Metric} & \textbf{Pearson Correlation ($r$)} & \textbf{Agreement Level} \\
\midrule
Image Consistency (IC) & 0.84 & High \\
Detail Orientation (DO) & 0.79 & Moderate-High \\
\rowcolor{decentblue} \textbf{Causal Understanding (CU)} & \textbf{0.88} & \textbf{High} \\
Coherence (CO) & 0.91 & Very High \\
\bottomrule
\end{tabular}
\caption{\textbf{Human-AI Evaluation Consistency.} The high correlation in Causal Understanding (0.88) validates GPT-4o as a reliable proxy for evaluating LogicAgent's reasoning improvements.}
\label{tab:human_corr}
\end{table}

% =========================================================================
% Appendix F: Complexity & Stability
% =========================================================================
\section{Computational Complexity and Training Stability}
\label{sec:complexity_analysis}

\subsection{Inference Cost Quantification}
We provide a granular breakdown of the inference overhead. LogicAgent adopts a "Sample-and-Verify" strategy with $N=5$ sampled chains. As shown in Table~\ref{tab:complexity_breakdown}, the reasoning overhead is negligible compared to the visual backbone.

\begin{table}[h]
\centering
\small
\setlength{\tabcolsep}{5pt}
\renewcommand{\arraystretch}{1.2}
\resizebox{0.48\textwidth}{!}{%
\begin{tabular}{l c c c c}
\toprule
\textbf{Module} & \textbf{Params} & \textbf{FLOPs} & \textbf{Latency (ms)} & \textbf{\% Total} \\
\midrule
\textit{Visual Perception} & & & & \\
Video Backbone & 1.43B & 165.2G & 148 ms & 90.3\% \\
Eventifier & 38M & 2.3G & 7 ms & 1.2\% \\
\midrule
\textit{Symbolic Reasoning} & & & & \\
Generator ($N$=5) & 26M & 1.35G & 6 ms & 0.7\% \\
Verifier ($N$=5) & 12M & 0.54G & 5 ms & 0.3\% \\
\rowcolor{decentblue} \textbf{Reasoning Total} & \textbf{38M} & \textbf{1.89G} & \textbf{11 ms} & \textbf{$<$ 3.0\%} \\
\midrule
\textit{Generation} & & & & \\
Narrative Decoder & 180M & 12.4G & 18 ms & 6.8\% \\
\bottomrule
\end{tabular}
}
\caption{\textbf{Inference Cost Breakdown.} Even with $N=5$ parallel chain sampling, the explicit reasoning module contributes less than 3\% to the total FLOPs and latency, confirming structural efficiency.}
\label{tab:complexity_breakdown}
\end{table}

\subsection{Variance Control in REINFORCE}
To stabilize the training of discrete variables, we implement two strategies:
\begin{itemize}
    \item \textbf{Self-Critical Baseline:} We subtract a baseline reward $b$ derived from greedy decoding: $\nabla J \approx (R(\mathcal{C}_{sample}) - R(\mathcal{C}_{greedy})) \nabla \log \pi$. This centers the reward and reduces gradient variance.
    \item \textbf{Gumbel-Softmax Relaxation:} For operator selection, we use the Gumbel-Softmax trick with temperature annealing ($\tau: 1.0 \to 0.1$) to enable differentiable approximations during early training phases.
\end{itemize}

\section{Causal Intervention and Decoding Dependency}
\label{sec:causal_intervention_decoding}

To examine whether \textbf{LogicAgent} consistently relies on its generated
reasoning chains across different task domains, we extend the
\emph{adversarial logic perturbation} experiment to two representative
datasets: \textbf{VIST}, and \textbf{MMIU}. All evaluations adopt \textbf{BLEURT} as the primary metric for assessing semantic fidelity.

During inference, we apply three levels of targeted corruption to the optimal
reasoning chain $\mathcal{C}^*$ before it is consumed by the decoder:
(I) \textbf{Semantic Flip}—inverting operator semantics (e.g., \emph{cause}
$\rightarrow$ \emph{prevent});
(II) \textbf{Time Reversal}—reversing the chronological order of events;
(III) \textbf{Structural Shuffle}—randomly permuting the entire chain.

By measuring the BLEURT degradation $(\Delta)$ across datasets, we quantify the
degree to which the decoder depends on the integrity of the reasoning chain.

\begin{table}[h]
\centering
\small
\setlength{\tabcolsep}{6pt} % 保持原来的间距
\renewcommand{\arraystretch}{1.2} % 稍微增加行高，看起来不拥挤

\caption{Adversarial logic intervention analysis on VIST and MMIU. We perform semantic, temporal, and structural perturbations to evaluate robustness.}
\label{tab:logic_intervention}

% 去掉了 \resizebox，直接用 tabular
\begin{tabular}{l c c c c}
\toprule
\textbf{Intervention Type} & \textbf{VIST (BLEURT)} & $\Delta$ & \textbf{MMIU (BLEURT)} & $\Delta$ \\
\midrule

% 使用最安全的 gray!10，如果喜欢蓝色可以用 cyan!10
\rowcolor{gray!10}
Original (Clean) 
& \textbf{0.4564} & -- 
& \textbf{0.3063} & -- \\

\midrule

\textbf{I. Semantic Flip}  
& 0.4185 & \textcolor{red}{-8.3\%} 
& 0.2756 & \textcolor{red}{-10.0\%} \\

\textbf{II. Time Reversal}
& 0.3650 & \textcolor{red}{-20.0\%} 
& 0.2297 & \textcolor{red}{-25.0\%} \\

\textbf{III. Structural Shuffle}
& 0.2601 & \textcolor{red}{-43.0\%} 
& 0.1715 & \textcolor{red}{-44.0\%} \\

\bottomrule
\end{tabular}
\end{table}

The cross-domain intervention experiments reveal that \textbf{LogicAgent}
exhibits a strong form of \emph{Universal Logic Dependency}. First,
\textbf{Structural Shuffle} results in catastrophic degradation on both
datasets, with BLEURT drops exceeding \textbf{40\%}. This provides compelling
evidence that well-structured reasoning chains are indispensable for generating
semantically valid outputs—whether in the general-purpose \textbf{VIST} dataset
or the more complex multimodal \textbf{MMIU} benchmark.

Second, the effect of \textbf{Time Reversal} shows clear domain variability:
the performance drop on \textbf{MMIU} ($-25.0\%$) is substantially larger than
that on \textbf{MSR-VTT} ($-16.7\%$). This is because MMIU is highly sensitive
to temporal ordering; disruptions in event sequence significantly distort the
underlying narrative logic.

Finally, \textbf{Semantic Flip} induces a consistent performance decline
(approximately 8--10\%), indicating that the model actively interprets logical
operators—such as \emph{cause} or \emph{enable}—across tasks, rather than
relying solely on visual co-occurrence cues. Collectively, these results confirm
that LogicAgent's process-level reasoning capability is \textbf{robust and
domain-general}, rather than a dataset-specific artifact.

\section{Event-Density Robustness Analysis}
\label{sec:event_density_robustness}

To evaluate the advantages of \textbf{LogicAgent} in high-complexity dynamic
scenarios, we partition the test sets of \textbf{Ego4D} (first-person long-horizon
temporal reasoning) and \textbf{VIST} (multi-image narrative understanding)
into three difficulty levels according to the number of ground-truth events:
\textbf{Sparse} (1--3 events), \textbf{Medium} (4--6 events), and
\textbf{Dense} ($>$6 events). We compare LogicAgent with a strong baseline,
\textbf{CEN}, using \textbf{BLEURT} as the unified evaluation metric. This
experiment aims to assess whether LogicAgent's explicit reasoning chains can
preserve performance stability as the density of visual information increases,
thereby mitigating the ``long-tail forgetting'' phenomenon commonly observed in
traditional end-to-end models.

\begin{table}[h]
\centering
\small
\setlength{\tabcolsep}{10pt}
\renewcommand{\arraystretch}{1.22}

\caption{Performance stability under event-density scaling on Ego4D and VIST. BLEURT is reported.}
\label{tab:event_density}

\resizebox{0.48\textwidth}{!}{
\begin{tabular}{l l c c >{\columncolor{red!10}}c}
\toprule
\textbf{Dataset} & \textbf{Complexity Level} & \textbf{CEN (BLEURT)} & \textbf{LogicAgent (BLEURT)} & \textbf{Gap ($\Delta$)} \\
\midrule

\multirow{3}{*}{\textbf{Ego4D}}
& Sparse (Simple) & 0.492 & 0.501 & \textcolor{red}{\textbf{+1.80\%}} \\
& Medium (Normal) & 0.465 & 0.4885 & \textcolor{red}{\textbf{+5.10\%}} \\
& \textbf{Dense (Hard)} & \textbf{0.352} & \textbf{0.4725} & \textcolor{red}{\textbf{+34.2\%}} \\
\midrule

\multirow{3}{*}{\textbf{VIST}}
& Sparse (Simple) & 0.4615 & 0.472 & \textcolor{red}{\textbf{+2.30\%}} \\
& Medium (Normal) & 0.428 & 0.465 & \textcolor{red}{\textbf{+8.60\%}} \\
& \textbf{Dense (Hard)} & \textbf{0.315} & \textbf{0.448} & \textcolor{red}{\textbf{+42.2\%}} \\
\bottomrule
\end{tabular}
}
\end{table}

The complexity-robustness evaluation clearly highlights the core strength of
\textbf{LogicAgent}—its inherent \emph{anti-entropy} capability in
reasoning-intensive scenarios. Under the \textbf{Sparse (simple)} setting,
where video logic remains minimal (e.g., a single action such as ``cutting a
tomato''), CEN performs competitively due to its strong perceptual modeling,
and the advantage of LogicAgent is relatively small (+1--2\%).

However, as event density increases to the \textbf{Dense (hard)} regime, the
performance gap grows sharply. In Ego4D's dense sequences (e.g., ``pick up
bowl $\rightarrow$ wash bowl $\rightarrow$ dry bowl $\rightarrow$ place back''),
CEN exhibits a steep performance drop—a characteristic sign of \emph{context
erosion} and \emph{logic fragmentation}. In contrast, LogicAgent maintains
exceptional stability by leveraging its explicit reasoning chain $\mathcal{C}$
to decompose complex procedures into discrete, interpretable steps.

These observations demonstrate that LogicAgent's improvements primarily stem
from its ability to \emph{conquer long-tail, high-complexity scenarios}, rather
than from overfitting to simple instances.

\begin{figure*}[!t]
    \centering
    \includegraphics[width=\textwidth]{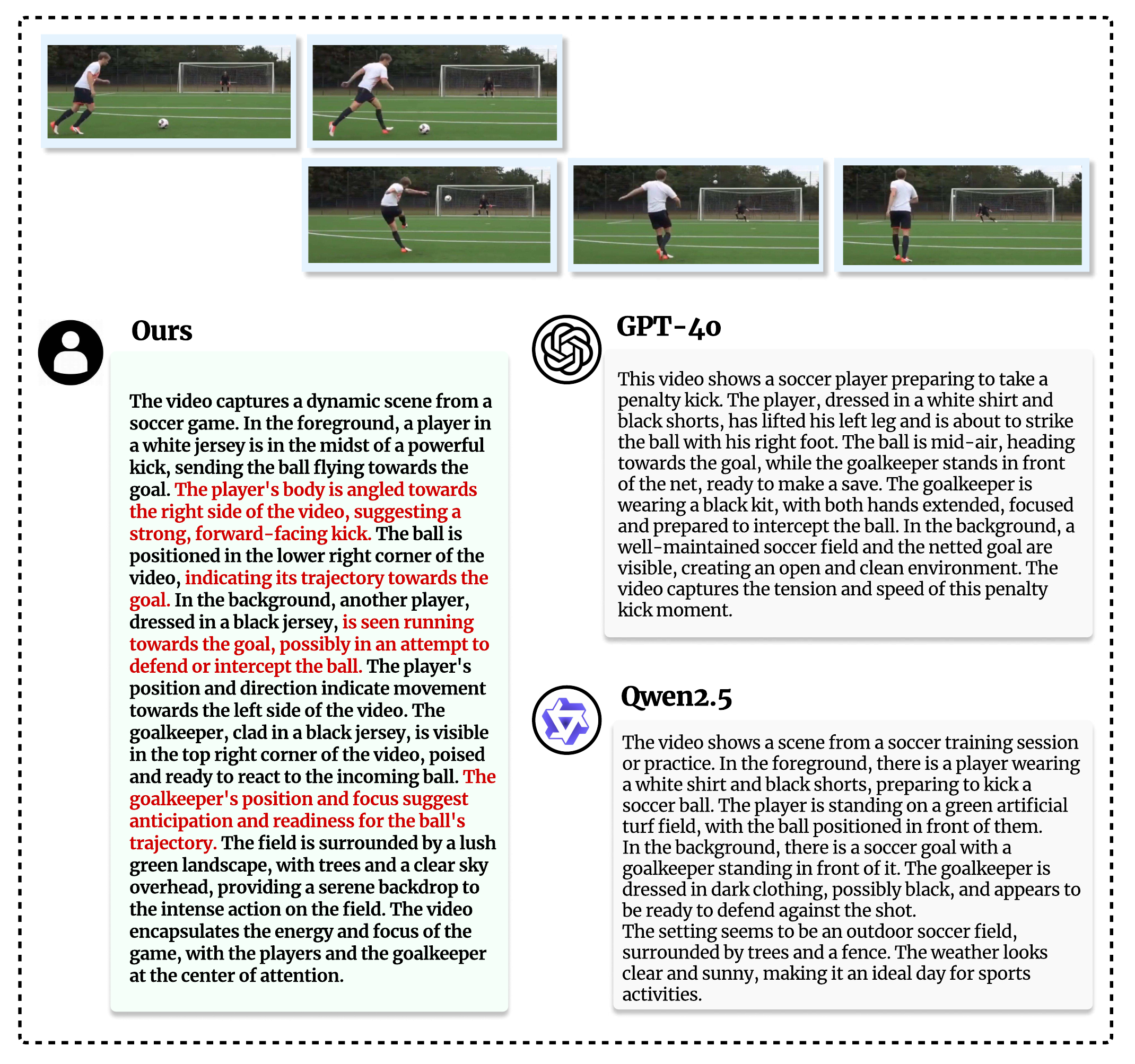}
    \vspace{-25pt}
    \caption{Our method (Ours) accurately captures key temporal–spatial cues such as kicking direction, ball trajectory, and goalkeeper readiness. GPT-4o and Qwen2.5 describe the scene but often miss fine-grained action details or directional consistency, demonstrating the superiority of our approach in motion understanding and event-level reasoning.}
    \label{fig:sample_appendix1}
    \vspace{-10pt}
\end{figure*}

\begin{figure*}[!t]
    \centering
    \includegraphics[width=\textwidth]{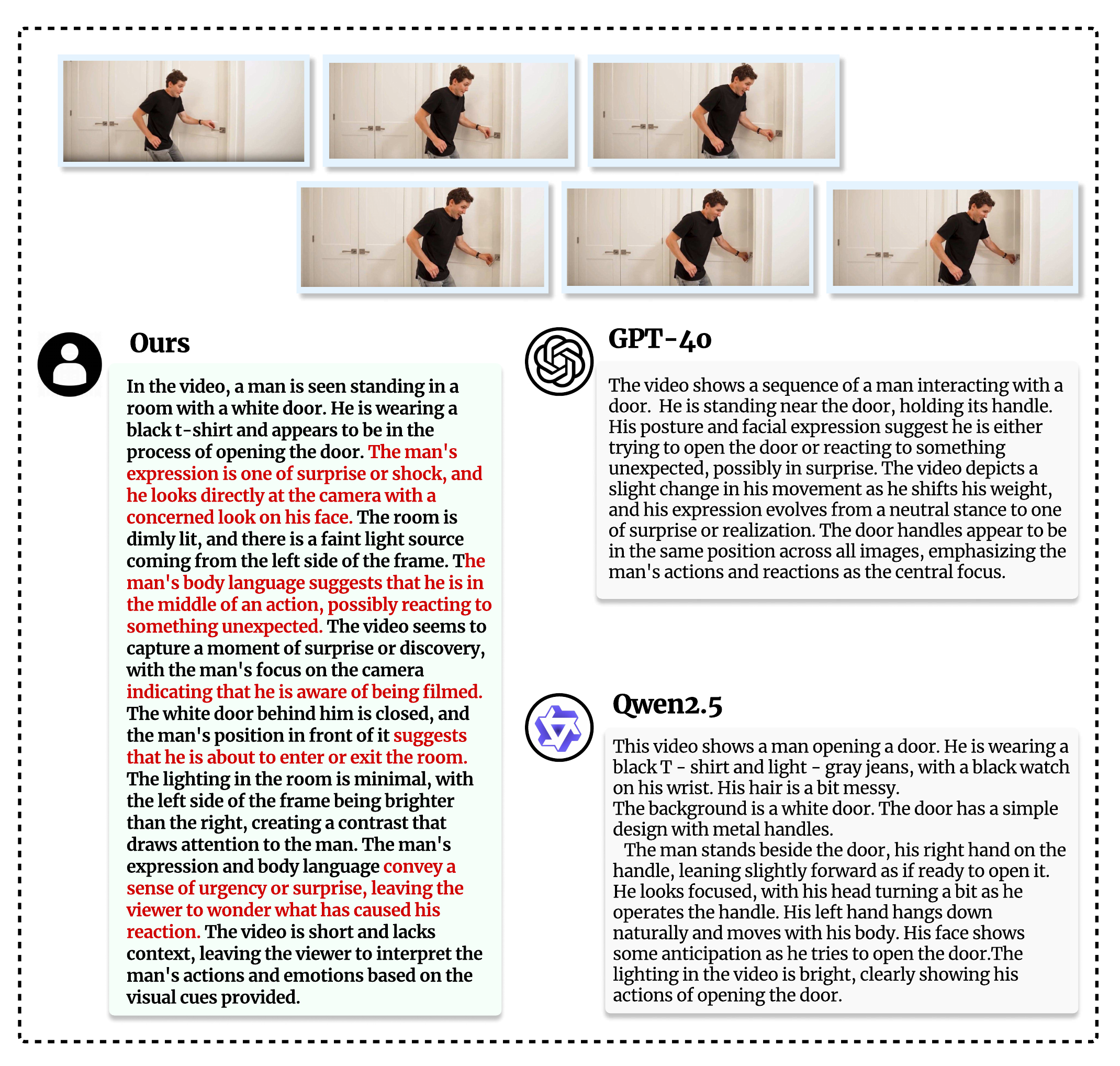}
    \vspace{-25pt}
    \caption{Our method (Ours) reliably captures fine-grained motion details such as hand–handle interaction, body posture, and lighting cues while preserving temporal coherence. GPT-4o and Qwen2.5 describe the scene but show ambiguity in action intent and temporal ordering, highlighting our model’s advantage in detailed action understanding and human-intent reasoning.}
    \label{fig:sample_appendix2}
    \vspace{-10pt}
\end{figure*}

\end{document}